\documentclass[a4paper,11pt]{article}
\usepackage[margin=1in]{geometry}

\usepackage[utf8]{inputenc} 
\usepackage[T1]{fontenc}    
\usepackage{hyperref}
\usepackage{url}            
\usepackage{booktabs}       
\usepackage{amsfonts,amsmath,amsthm}       
\usepackage{nicefrac}       
\usepackage{microtype}      
\usepackage{xcolor}         
\usepackage{csquotes}
\usepackage{authblk}
\usepackage[square,numbers]{natbib}

\usepackage{graphicx}
\usepackage{wrapfig}
\usepackage{relsize}
\usepackage{color}
\usepackage{pict2e}
\usepackage{subcaption}
\usepackage{algorithm}
\usepackage[noend]{algorithmic}
\usepackage{caption}
\usepackage{nameref}
\usepackage{makecell}
\usepackage[font={small}]{caption}

\usepackage{enumitem}

\newcommand{\E}{\mathbb{E}}
\newcommand{\Var}{\mathrm{Var}}
\newcommand{\R}{\mathbb{R}}
\newcommand{\Sph}{\mathbb{S}}
\newcommand{\norm}[1]{\left\lVert#1\right\rVert_2}
\newcommand{\norminf}[1]{\left\lVert#1\right\rVert_\infty}

\newcommand{\cL}{\mathcal{L}}
\newcommand{\cO}{\mathcal{O}}
\newcommand{\MSEMA}{\text{MSE}_{\text{MA}}}
\newcommand{\coloneqq}{:=}
\newcommand{\cF}{\mathcal{F}}

\newcommand{\Lw}{\mathcal{L}^{(w)}}

\newtheorem{theorem}{Theorem}[section]
\newtheorem{lemma}[theorem]{Lemma}
\newtheorem{corollary}[theorem]{Corollary}
\newtheorem{definition}[theorem]{Definition}
\newtheorem{assumption}[theorem]{Assumption}

\newtheorem{remark}[theorem]{Remark}

\title{\textbf{SSTQ:Privacy-Preserving Vector Quantization via
Subsampled Stochastic TurboQuant}}
\author[1,3]{Adel Javanmard}
\author[2,3]{David P. Woodruff }
\author[3]{Vahab Mirrokni}
\affil[1]{University of Southern California, \texttt{ajavanma@usc.edu}}
\affil[2]{Carnegie Mellon University, \texttt{dwoodruf@andrew.cmu.edu}}
\affil[3]{Google Research, \texttt{mirrokni@google.com}}
\date{}

\begin{document}

\maketitle

\begin{abstract}
Achieving local differential privacy in distributed optimization while maintaining low communication cost remains challenging. Existing vector quantization methods, such as vqSGD, rely on high-dimensional geometric constructions but incur unfavorable dimension-dependent variance. In this work, we propose Subsampled Stochastic TurboQuant (SSTQ), a framework that combines a bounded Kashin representation, data-independent coordinate subsampling, and privacy-aware one-dimensional quantization. SSTQ includes two variants: (1) a Flat Randomized Response variant that is unbiased and, for a fixed codebook bit-width, frame redundancy, and dimension-independent Kashin level, achieves reconstruction MSE that scales linearly with the ambient dimension $d$, while using only \( \lceil \log_2 N \rceil + b \) bits per message. Here, \(N = \Theta(d)\) denotes the frame size in the Kashin transform and $b$ is the codebook bit-width; and (2) a metric-aware truncated-Laplace variant that removes the exponential dependence on bit-width at the cost of a non-vanishing bias. We also derive a convex uniform-surrogate codebook objective whose worst-case codebook-dependent upper bound improves from $\mathcal{O}(4^b)$ to $\mathcal{O}(2^b)$. Experiments on synthetic regression, Fashion-MNIST, and CIFAR-10 compare the per-message privacy–utility and uplink-communication trade-offs of SSTQ with those of established baselines, demonstrating favorable utility and communication efficiency.
\end{abstract}

\section{Introduction and Motivation}

The communication bottleneck in federated learning (FL) and large-scale distributed optimization has led to significant interest in gradient compression techniques~\cite{li2020federated,kairouz2021advances}. At the same time, protecting sensitive client data against adversarial or untrusted aggregators necessitates the use of Local Differential Privacy (LDP)~\cite{dwork2006calibrating,kasiviswanathan2011can,warner1965randomized}. These two objectives—communication efficiency and strong privacy guarantees—are inherently at odds. Compression reduces high-dimensional continuous signals to compact discrete representations, while LDP requires adding noise and maintaining sufficient support over the same space to ensure plausible deniability.

Existing approaches to private vector release can be broadly understood through three underlying paradigms, each with distinct limitations. Geometric quantization methods, such as vqSGD~\cite{gandikota2021vqsgd}, encode vectors using structured high-dimensional polytopes and apply randomized response over the resulting discrete set. Although communication-efficient, their reliance on high-dimensional geometry leads to unfavorable variance scaling, which can grow cubically with the dimension. In Appendix~\ref{app:curse}, we show that their mean-squared error (equal to the variance because the estimators are unbiased) is $\Theta_\epsilon(d^3)$ for fixed $\epsilon$, and $\Theta(d^3/\epsilon^2)$ uniformly in the high-privacy regime $0<\epsilon\le1$.

A second class of methods achieves optimal variance scaling under LDP. Mechanisms such as SQKR~\cite{chen2020breaking} and PrivUnit~\cite{bhowmick2018protection} attain the information-theoretic limit of $O(d/\epsilon^2)$ by operating over bounded representations, including Kashin frames or continuous distributions. However, these approaches either impose restrictive communication constraints—such as 1-bit quantization and shared randomness in SQKR—or require transmitting dense, uncompressed vectors, as in PrivUnit, limiting their practicality in bandwidth-constrained settings.

A third class of methods consists of data-oblivious quantization schemes that combine randomized orthogonal transformations with scalar codebooks chosen independently of the dataset. This template avoids data-dependent codebook training while attaining strong compression performance across a broad range of dimensions and communication rates. Early examples include DRIVE and EDEN, which develop rotation-based schemes for distributed mean estimation with unbiased reconstruction and flexible communication rates \cite{vargaftik2021drive,vargaftik2022eden}. More recently, TurboQuant~\cite{zandieh2025turboquant} demonstrates that this approach can achieve near-optimal distortion over a wide range of codebook bit-widths and dimensions. Related applications include RaBitQ for distance and inner-product estimation \cite{gao2024rabitq}. While these methods are highly effective for compression, they do not provide privacy guarantees and can suffer from directional bias when combined with naive privatization schemes. In addition, the TurboQuant pipeline first applies a random rotation to the input vectors and then exploits the near-independence of coordinates in high-dimensional spaces to apply optimal scalar quantization to each coordinate. Consequently, the number of output bits scales on the order of the input dimension, corresponding to a constant bit-width per coordinate. In contrast, the present focus is on transmitting a compressed representation using a number of bits that grows logarithmically with the dimension.

In this work, we build on these insights and develop a data-oblivious, differentially private quantization framework based on one-dimensional scalar quantization of Kashin representations. We derive the corresponding mean squared error and design codebooks that are optimized for the privatized setting, enabling efficient communication while preserving strong privacy guarantees. \emph{Although our goal differs, we emphasize this connection by incorporating ``TurboQuant'' directly into the name of our procedure.}

Multiple correlation-aware compression schemes have also been proposed where the focus is to exploit the potential similarities between vectors whose mean is to be estimated in a federated learning setting~\cite{suresh2022correlated,jiang2023correlation,vardhan2026collaborative}, as well as low-complexity decoding/encoding procedures~\cite{ben2024accelerating}.


\begin{table}[htpb]
\begin{centering}
\caption{Single-message reconstruction MSE for normalized inputs $x\in B_2^d(1)$. Constants depending on fixed frame parameters are suppressed. SQKR's displayed risk is the per-user contribution to mean-estimation risk before division by the number of users.}
\label{tab:baselines}
\vspace{0.1cm}
\resizebox{\textwidth}{!}{
\begin{tabular}{llll}
\toprule
\textbf{Mechanism} & \textbf{Topological / DP Geometry} & \textbf{Number of Bits per Client} & \textbf{Pure LDP MSE}  \\
\midrule
\textbf{vqSGD (cross-polytope RR)} & Cross-polytope & $\lceil\log_2 d\rceil + 1$ & $d\!\left(\frac{e^\epsilon+2d-1}{e^\epsilon-1}\right)^2-1$ \\
\textbf{SQKR} & Binary Kashin frame & $k$ public-coin; $k(\lceil\log_2N\rceil+1)$ private-coin & $\cO\!\left(\frac{d}{\epsilon^2\wedge\epsilon\wedge b_0}\right)$ \\
\textbf{PrivUnit} & Continuous sphere & continuous output; $32d$ under float32 & $\cO\!\left(\frac{d}{\epsilon\wedge\epsilon^2}\right)$ \\
\textbf{SSTQ (Flat-RR)} & Scalar Kashin codebook & $\lceil\log_2N\rceil+b$ & $\cO\!\left(\frac{dK^2}{\rho}\left[1+\frac{4^b}{\epsilon\wedge\epsilon^2}\right]\right)$\\
\textbf{SSTQ (Flat-RR, optimized surrogate)} & Scalar Kashin codebook & $\lceil\log_2N\rceil+b$ & $\cO\!\left(\frac{dK^2}{\rho}\left[1+\frac{2^b}{\epsilon\wedge\epsilon^2}\right]\right)$\\
\textbf{SSTQ (Metric-Aware)} & Scalar Kashin codebook & $\lceil\log_2N\rceil+b$ & $\cO\!\left(\frac{dK^2}{\rho}\left[1+\frac{C^2}{(2^b-1)^2}+\frac{1}{\epsilon\wedge\epsilon^2}\right]\right)$\\
\bottomrule
\end{tabular}
}
\end{centering}
{\scriptsize Here $\rho=N/d$. For SQKR, $k=\min\{\lceil\epsilon\log_2e\rceil,b_0\}$. A public coin refers to randomness shared in advance between the client and server, while a private coin is randomness generated locally by the client and should be communicated to the server for  reconstruction.} 
\vspace{0.1cm}
\end{table}

\vspace{-0.3cm}

\subsection{Our Contributions}
We present a unified framework connecting information-theoretic quantization with local differential privacy. Our main contributions are as follows:

\begin{itemize}[leftmargin=*]
    \item \textbf{SSTQ architecture:}  
    We introduce Subsampled Stochastic TurboQuant (SSTQ), which combines a scaled Parseval frame admitting a bounded Kashin representation with data-independent coordinate subsampling. For a fixed codebook bit-width $b$, fixed redundancy (frame aspect ratio) $\rho= N/d$, and  
dimension-independent Kashin constant $K$, 
we show that this structure achieves an MSE scaling of \( \mathcal{O}(d/(1\wedge\epsilon\wedge\epsilon^2)) \).  This is the optimal rate in the constant-$\epsilon$ regime~\cite{duchi2019lower}, while requiring only \( \lceil \log_2 N \rceil + b \) uplink bits in the private-coin setting. 
    
    \item \textbf{Privacy-aware codebook optimization:}  
    We derive a surrogate privacy-aware stochastic quantization loss \( \mathcal{L}_{\text{SSTQ}} \), yielding a convex formulation for codebook design. Optimizing this loss improves the codebook-dependent worst-case  MSE scaling from $\mathcal{O}(4^b)$ (Theorem~\ref{thm:miracle}) to $\mathcal{O}(2^b)$ (Theorem~\ref{thm:one_third_bound}).
    
    \item \textbf{Metric-aware mechanism for dense regimes:}  
    We propose a metric-aware truncated-Laplace mechanism that avoids the exponential variance growth of standard randomized response at higher codebook bit-widths, while preserving pure \( \epsilon \)-LDP and enabling efficient multi-bit quantization.
    
    \item \textbf{Empirical evaluation:}  
    We evaluate SSTQ (Flat-RR and Metric-Aware) against established baselines, including PrivUnit, vqSGD, and SQKR, in federated learning tasks on CIFAR-10 and Fashion-MNIST. Experiments are conducted across varying dimensions and codebook sizes to assess utility, variance behavior, and communication efficiency.
\end{itemize}

\vspace{-.3cm}

\subsection{Distinctions from Prior Kashin-Based LDP Mechanisms}

Overcomplete tight frames (Kashin representations) have been used in prior work, notably SQKR~\cite{chen2020breaking}, to control sensitivity under joint privacy and communication constraints. While SQKR achieves the optimal \( \mathcal{O}(d/(\epsilon\wedge\epsilon^2)) \) mean squared error scaling, its design imposes several structural limitations. SSTQ differs in three key aspects:

First, SSTQ removes the effective bandwidth restriction present in SQKR. In SQKR, dense encoding requires transmitting one bit per Kashin coefficient, leading to \( \Theta(d) \) bits. In the subsampled setting, grouping \(k\) coordinates with Flat Randomized Response incurs an exponential variance penalty \( \mathcal{O}(4^k) \) (see Claim C.1 in ~\cite{chen2020breaking}), which forces \(k\) to be capped by \( \min(\lceil \epsilon\log_2 e \rceil, b_0) \). In contrast, SSTQ allocates the full \(b\)-bit budget to a single coordinate. In its metric-aware  variant, it operates over a one-dimensional codebook, thereby fully utilizing the available bandwidth without incurring exponential penalties.

Second, SSTQ replaces 1-bit quantization with a continuous surrogate optimization. SQKR maps Kashin coefficients to their extreme values prior to privatization, which fixes the baseline variance at a high level. SSTQ instead performs stochastic quantization over a multi-bit codebook optimized for the privatized setting, effectively reducing quantization error while accommodating LDP noise.

Third, SSTQ operates as a private-coin protocol with explicit communication cost. SQKR relies on shared randomness to coordinate subsampling; without it, transmitting indices increases the communication cost to \( \mathcal{O}(k \log d) \) to send the indices of the $k$ sampled coordinates. SSTQ directly transmits the sampled index along with the quantized value, resulting in a fixed cost of \( \lceil \log_2 N \rceil + b \) bits, independent of shared randomness assumptions.
\vspace{-0.3cm}

\section{Related Work and Mathematical Preliminaries}
\subsection{Related Work}
\label{sec:related_work}
The fundamental tension between Local Differential Privacy (LDP) and extreme communication bandwidth in distributed mean estimation has motivated a rich recent literature. Our Subsampled Stochastic TurboQuant (SSTQ) framework differs from prior approaches along three axes: privacy threat model, quantization efficiency, and geometric optimization.

\textbf{Geometric Quantization and Dimensional Curses.} Foundational mechanisms such as PrivUnit and PrivUnitG \cite{bhowmick2018protection,asi2022optimal} achieve the information-theoretic optimal LDP variance of \(\mathcal{O}(d/\epsilon^2)\) but require transmitting uncompressed dense continuous vectors. By contrast, geometric quantizers such as vqSGD \cite{gandikota2021vqsgd} reduce bandwidth logarithmically by mapping vectors to high-dimensional cross-polytopes. However, this rigid geometry shrinks the LDP probability gap and triggers the \(\Theta(d^3/\epsilon^2)\) variance curse proven in Appendix~\ref{app:curse}. While recent work has broken dimension dependence for \textit{sparse} discrete distribution estimation under communication constraints \cite{chen2021breaking}, SSTQ achieves a similar breakthrough for dense, high-dimensional continuous gradient quantization.

\textbf{Random Projections and Subspace LDP.} To reduce the \(\mathcal{O}(d)\) communication bottleneck of optimal LDP mechanisms, methods such as ProjUnit \cite{asi2023fast} project gradients into a lower \(k\)-dimensional subspace before applying continuous LDP. Although ProjUnit achieves order-optimal expected MSE, it still outputs continuous vectors that require further quantization and relies on dense matrix-vector multiplications during decoding. SSTQ instead natively unifies LDP with discrete quantization: by coupling scaled Parseval frames with 1-sparse subsampling, it reduces the problem to exact 1D scalar quantization, with payload \(\lceil \log_2 N \rceil + b\) bits.

\textbf{Central/Shuffle DP and Coordinate Subsampling.} Coordinate subsampling has also been used to amplify privacy under relaxed trust models. The Coordinate Subsampled Gaussian Mechanism (CSGM) \cite{chen2023privacy} combines Kashin's representation with random subsampling to achieve optimal \(\mathcal{O}(d/n^2\epsilon^2)\) error for the estimation of the aggregated mean across $n$ clients under Central/Shuffle DP, but it depends on a trusted central server or a secure shuffler. SSTQ instead operates in the strictly stronger pure \(\epsilon\)-Local DP setting, providing message-level plausible deniability without trusted aggregation.

\textbf{Geometric Constants and Continuous Optimization.} The streaming variant \(L_2\)-CSGM \cite{chen2024improved} shows that pushing \(L_2\) geometries through \(L_\infty\) Kashin bounds with rigid 1-bit quantization, as in SQKR \cite{chen2020breaking}, leads to suboptimal MSE constants. SSTQ demonstrates that this limitation stems from the rigid binarization, not from the overcomplete frame itself: by replacing 1-bit quantization with a multi-bit continuous surrogate geometry optimized via \(\mathcal{L}_{\text{SSTQ}}\), it factors out these spatial penalties and reduces the codebook-dependent MSE scaling from $\mathcal{O}(4^b)$ to $\mathcal{O}(2^b)$ (Theorem~\ref{thm:one_third_bound}), all within pure LDP.

\textbf{Joint Quantization--Privacy Co-Design.} A complementary line of work co-designs
compression and privatization rather than composing them. JoPEQ~\citep{lang2023joint} applies
universal vector quantization over a shared random lattice with subtractive dithering, so that
the quantization distortion is statistically equivalent to additive noise, and augments this
distortion with dedicated multivariate Laplace-type noise to meet a target LDP budget, whereas
SoteriaFL~\citep{li2022soteriafl} pairs a general family of (shifted) compression operators with
Gaussian-mechanism privatization and characterizes the resulting privacy--utility--communication
trade-offs, improving on the naive CDP-SGD composition of compression with DP-SGD. Both are
complementary to our setting but differ in threat model and rate: SoteriaFL's Gaussian noise
yields the $(\epsilon,\delta)$ relaxation rather than pure $\epsilon$-LDP, while JoPEQ's
subtractive dither presumes a lattice shared between client and server, a public-coin
protocol, with a rate specified per entry and hence $\Theta(d)$ uplink bits. SSTQ instead
privatizes a \emph{finite} label alphabet under pure $\epsilon$-LDP in the private-coin model at a
fixed cost of $\lceil\log_2 N\rceil+b$ bits per message, and, unlike a lattice or compression
operator fixed a priori, optimizes the codebook geometry itself against the privatized MSE
(Section~\ref{sec:opt-code}).
\vspace{-0.3cm}

\subsection{Mathematical Preliminaries}
\label{sec:preliminaries}

Before presenting the algorithmic developments, we briefly review the basic notions of local differential privacy and overcomplete frame theory. 
\medskip

\noindent{\bf Local Differential Privacy (LDP).} LDP requires each client to perturb its data before transmission, ensuring that the released message remains private even against an untrusted aggregator. In particular, the server cannot reliably infer the original client data, regardless of any auxiliary information or post-processing. The formal definition is as follows.
\begin{definition}[$\epsilon$-Local Differential Privacy] \label{def:ldp}
A randomized mechanism $\mathcal{M} : \mathcal{X} \to \mathcal{Z}$ satisfies pure $\epsilon$-Local Differential Privacy ($\epsilon$-LDP) if for all pairs of possible inputs $x, x' \in \mathcal{X}$ and for all measurable subsets of outputs $S \subseteq \mathcal{Z}$:
\begin{equation}
    \Pr[\mathcal{M}(x) \in S] \le e^\epsilon \Pr[\mathcal{M}(x') \in S]
\end{equation}
where $\epsilon > 0$ is the strict privacy budget.
\end{definition}

\medskip


\begin{definition}[Scaled Parseval tight frame]\label{def:frame}
Let $U\in\R^{N\times d}$, with rows $u_1^\top,\ldots,u_N^\top$, satisfy
\[
    U^\top U=\frac Nd I_d.
\]
Equivalently,
\[
    \Phi:=\sqrt{\frac dN}\,U^\top\in\R^{d\times N}
    \qquad\text{satisfies}\qquad
    \Phi\Phi^\top=I_d,
\]
so $\Phi$ is a Parseval frame. We call $U$ the corresponding scaled Parseval frame.
\end{definition}

\begin{lemma}[Aggregate frame energy]\label{lem:trace}
Every scaled Parseval frame satisfies
\[
    \sum_{j=1}^N \norm{u_j}^2
    =\operatorname{Tr}(U^\top U)=N.
\]
\end{lemma}

\begin{theorem}[Kashin's Representation Theorem]\label{thm:kashin_rep}
Let $\rho=N/d>1$, let $U\in\R^{N\times d}$ be a scaled Parseval frame, and suppose that the associated Parseval frame $\Phi=\sqrt{d/N}\,U^\top$ satisfies the Lyubarskii--Vershynin $(\eta,\delta)$-uncertainty principle~\cite{lyubarskii2010uncertainty}, where $0<\eta,\delta<1$. Then every $x\in\R^d$ admits a coefficient vector $y\in\R^N$ such that
\[
    x=\frac dN U^\top y,
    \qquad
    \norminf{y}\le \frac{K}{\sqrt N}\norm{x},
    \qquad
    K=\frac{\sqrt{\rho}}{(1-\eta)\sqrt{\delta}}.
\]
\end{theorem}

The theorem follows from the Parseval-frame formulation in~\cite{lyubarskii2010uncertainty} after rescaling the coefficients. Thus $K$ depends on the redundancy and on the uncertainty-principle constants. 

For any fixed redundancy $\rho = N/d>1$, a Haar-distributed $N\times d$ Stiefel matrix, scaled so that $U^\top U = (N/d) I_d$, is a tight frame and satisfies the uncertainty principle with high probability with constants depending only on $\rho$. Consequently, its Kashin level is independent of $d$.  
Structured partial Fourier or Hadamard constructions permit fast multiplication, but the available guarantees give a polylogarithmic rather than constant Kashin level. Therefore, we retain the dependence on $K$ explicitly in presenting our bounds.

\section{Subsampled Stochastic TurboQuant (SSTQ)}\label{sec:SSTQ-FR}
As discussed in Appendix~\ref{app:curse}, prior work on geometric vector quantization, most notably vqSGD~\cite{gandikota2021vqsgd}, suffers from a $\Theta(d^3/\epsilon^2)$ variance growth rate, which severely limits its effectiveness for high-dimensional gradient embeddings. To address this limitation, we propose Subsampled Stochastic TurboQuant (SSTQ), a framework that decouples the codebook size $M$ from the ambient dimension $d$ by combining overcomplete tight frames, data-oblivious coordinate subsampling, and one-dimensional stochastic quantization.
\medskip

\noindent\textbf{The SSTQ Algorithmic Pipeline.}
Fix a public input radius $R>0$ and let $\mathcal X=B_2^d(R)=\{x\in\R^d:\norm{x}\le R\}$. Let $N=\rho d>d$, where $\rho$ is an admissible frame redundancy, and let $U$ be a public scaled Parseval frame satisfying Theorem~\ref{thm:kashin_rep}. The corresponding bound on the Kashin coefficient is given by
$B=\frac{KR}{\sqrt N}$. All protocol parameters, including $R,U$, and the codebook, are fixed independently of the protected input.

We represent the scalar codebook as an indexed list \(\Gamma=(c_1,\ldots,c_M)\), with $M=2^b$, rather than as a set. We refer to $b = \log_2(M)$ as the \emph{``codebook bit-width''}. The labels $1,\ldots,M$ remain distinct even when some scalar values coincide. We assume
\[
    -B=c_1\le c_2\le\cdots\le c_M=B,
    \qquad
    \sum_{i=1}^M c_i=0,
\]
where repeated entries are counted with multiplicity. 


 Algorithm~\ref{alg:encode} takes a codebook $\Gamma$ as input; the design of this codebook is discussed later in Section~\ref{sec:opt-code}, where we derive a surrogate privacy-aware stochastic quantization loss to guide codebook optimization. The algorithm currently uses Flat Randomized Response. In Section~\ref{sec:discrete_laplace}, we introduce an alternative variant based on a Metric-Aware truncated-Laplace mechanism, which replaces the flat randomized response with continuous truncated Laplace noise injection followed by nearest-codeword quantization, and is better suited to higher codebook bit-width regimes.

 Note that because the protocol employs a 1-sparse projection over an $N$-dimensional Kashin frame, compressing the full $d$-dimensional target vector requires transmitting exactly two discrete components from client to the server: the coordinate index $J\in [N]$ (requiring $\lceil \log_2 N \rceil$ bits) and the  randomized codeword index $Z\in [M]$ (requiring $b$ bits). Consequently, the communication payload per message evaluates to $\lceil \log_2 N \rceil + b$ bits. Notably, this scheme does not rely on shared randomness (i.e., public-coin assumptions), in contrast to SQKR~\cite{chen2020breaking}. 
\medskip

\noindent\textbf{Theoretical Guarantees.} We now state the  theoretical guarantees on the utility and privacy of the SSTQ pipeline. The proofs of these properties are deferred to Appendix~\ref{app:SSTQ-1}.

\begin{theorem}[Strict $\epsilon$-LDP and Unbiasedness]
\label{thm:sstq_ldp_unbiased}
The SSTQ with flat randomized response, as presented in Algorithm~\ref{alg:encode}, satisfies pure $\epsilon$-Local Differential Privacy. Furthermore, the decoding algorithm given by Algorithm~\ref{alg:decode} returns an unbiased estimator, $\E[\hat{x}\mid x] = x$.
\end{theorem}
Our next theorem bounds the mean-squared-error (MSE) of the SSTQ (with Flat-RR) algorithm.

\begin{theorem}[MSE of SSTQ (Flat-RR)] \label{thm:miracle}
Let $x\in B_2^d(R)$, $M=2^b$, and let $\Gamma$ be any indexed, ordered, boundary-anchored, zero-sum codebook supported on $[-B,B]$. Then, for the reconstructed vector $\hat{x}$ by the Flat-RR SSTQ algorithm, we have  
\begin{align}
    \operatorname{MSE}_{\mathrm{SSTQ}}(x):=\E\!\left[\norm{\hat x-x}^2\mid x\right]
    \le
    \frac{d^2B^2}{(p-q)^2}-\norm{x}^2 
    =
    \frac{dK^2R^2}{\rho}
    \left(1+\frac{M}{e^\epsilon-1}\right)^2-\norm{x}^2.
    \label{eq:flat-exact-mse-bound}
\end{align}
Consequently,
\[
    \operatorname{MSE}_{\mathrm{SSTQ}}(x)
    =
    O\!\left(
      \frac{dK^2R^2}{\rho}
      \left[1+\frac{4^b}{\epsilon\wedge\epsilon^2}\right]
    \right).
\]
\end{theorem}
Note that for constant codebook bit-width $b$ and constant-$\epsilon$, the SSTQ achieves the optimal minimax rate of $O(dR^2/(\epsilon\wedge\epsilon^2))$ for the private mean estimation (no-quantization), established by~\cite{duchi2019lower}. For $n$ independent clients
whose unbiased messages are averaged, the mechanism-induced contribution is divided by $n$.

\begin{remark}\label{rem:curse}
In non-private quantization, increasing the codebook bit-width $b$ reduces quantization error by refining the discrete grid of $M= 2^b$ states. Under local differential privacy, however, the mean squared error is affected both by quantization error and by the variance introduced through privatization. With Flat Randomized Response, a client’s true bin is replaced uniformly by one of the other $M-1$ without using the geometry of the codebook. As $M$ grows, this uniform randomization increasingly dilutes the signal and leads to larger reconstruction error. In particular, unbiased decoding under uniform Randomized Response requires a correction factor that scales with $M= 2^b$, which motivates the metric-aware randomization mechanism introduced in Section~\ref{sec:discrete_laplace}. 
\end{remark}
\begin{algorithm}[t]
\caption{SSTQ Client Encoding Algorithm (Flat-RR)}\label{alg:encode}
\begin{algorithmic}[1]
\REQUIRE $x\in B_2^d(R)$, public scaled Parseval frame $U\in\R^{N\times d}$, and public indexed codebook $\Gamma=(c_1,\ldots,c_M)$.
\STATE \textbf{Redundant Kashin Bounding:} Execute Lyubarskii's algorithm~\cite{lyubarskii2010uncertainty} to compute $y \in \R^N$ such that $\frac{d}{N} U^\top y = x$ and $\norminf{y} \le B = \frac{KR}{\sqrt{N}}$.
\STATE \textbf{Oblivious Subsampling:} Sample $J\sim\operatorname{Unif}([N])$, independently of $x$.
\STATE \textbf{Stochastic Interpolation on labels:} If $y_J=c_i$ for one or more indices, set
    \[V=\min\{i:c_i=y_J\}.\]
    Otherwise set $k_-=\max\{i:c_i<y_J\}$ and $k_+=\min\{i:c_i>y_J\}$, and sample $V\in\{k_-,k_+\}$ according to
    \[
    \Pr(V=k_+\mid y_J)=\frac{y_J-c_{k_-}}{c_{k_+}-c_{k_-}},\qquad
    \Pr(V=k_-\mid y_J)=\frac{c_{k_+}-y_J}{c_{k_+}-c_{k_-}}.
    \]
\STATE \textbf{Flat Randomized Response on labels:} Sample $Z\in[M]$ with
    \[
    \Pr(Z=\ell\mid V=i)=
    \begin{cases}
       p=\frac{e^\epsilon}{e^\epsilon+M-1},&\ell=i,\\[3pt]
       q=\frac{1}{e^\epsilon+M-1},&\ell\ne i.
    \end{cases}
    \]
\RETURN Transmit $(J,Z)$. \textbf{Cost:} $\lceil\log_2N\rceil+b$ bits.
\end{algorithmic}
\end{algorithm}

\begin{algorithm}[t]
\caption{SSTQ Server Decoding Algorithm (Flat-RR)}\label{alg:decode}
\begin{algorithmic}[1]
\REQUIRE Message $(J,Z)$, indexed codebook $\Gamma=(c_1,\ldots,c_M)$, and scaled Parseval frame $U$.
\STATE \textbf{Scalar unbiasing:} Set $\widetilde y_J=c_Z/(p-q)$.
\STATE \textbf{Sparse reconstruction:} Set $\widetilde y=N\widetilde y_Je_J\in\R^N$.
\STATE \textbf{Frame projection:} Set
    \[\hat x=\frac dN U^\top\widetilde y=d\,\widetilde y_Ju_J.\]
\RETURN $\hat x$.
\end{algorithmic}
\end{algorithm}


\section{Uniform-Surrogate Privacy-Aware Codebook Optimization ($\cL_{\text{SSTQ}}$)}\label{sec:opt-code}


A salient feature of SSTQ is that the Kashin representation spreads the signal across bounded scalar coordinates. By jointly accounting for stochastic interpolation and the variance induced by randomized response, we derive a data-oblivious, privacy-aware codebook optimization procedure that is well suited to online applications.

Let $Y$ have an arbitrary probability law $\mu$ supported on $[-B,B]$; $\mu$ can be discrete, as
it is for the empirical coordinate law of a fixed vector. For the indexed interpolation rule in Algorithm~\ref{alg:encode}, define
\[
r_\Gamma(y):= \Var(c_V|Y=y)\,.
\]
If $c_k<y<c_{k+1}$ across a positive gap, then $r_\Gamma(y) = (y-c_k)(c_{k+1}-y)$.
If $y$ equals a codeword value, then $r_\Gamma(y)=0$. Thus, $r_\Gamma$ remains well defined even when codeword values are repeated.
As shown in Appendix~\ref{sec:surrogate-derivation}, applying the law of total variance to interpolation followed by Flat RR gives
\begin{equation}
\Var(\widetilde y_j\mid Y=y)
=\frac{r_\Gamma(y)}{p-q}
 +\left(\frac{1}{p-q}-1\right)y^2
 +\frac{q}{(p-q)^2}\sum_{i=1}^M c_i^2.
\end{equation}
Since $\E[\widetilde y_j\mid Y=y]=y$, it follows that
\begin{equation}\label{eq:expected_mse}
\E[\widetilde y_j^2]
=\frac{1}{p-q}
\underbrace{\left[
  \int r_\Gamma(y)\,d\mu(y)
  +\frac{1}{e^\epsilon-1}\sum_{i=1}^M c_i^2
\right]}_{:=\cL^*_{\mathrm{SSTQ}}(\Gamma;\mu)}
+\frac{1}{p-q}\int y^2\,d\mu(y).
\end{equation}
It is straightforward to see that $\text{MSE}_{\text{SSTQ}} < d^2 \E[\tilde{y}_j^2]$ (see Equation~\eqref{eq:MSE-B}). Also,
the last integral above does not depend on $\Gamma$. This motivates the codebook-dependent loss for a general coordinate law given by $\cL^*_{\mathrm{SSTQ}}(\Gamma;\mu)$. If $\mu$ has a density $f_Y$, this loss can equivalently be written as
\begin{equation}\label{eq:exact-loss}
\cL^*_{\mathrm{SSTQ}}(\Gamma;\mu)
=\sum_{k:\,c_k<c_{k+1}}
  \int_{c_k}^{c_{k+1}}
  (y-c_k)(c_{k+1}-y)f_Y(y)\,dy
 +\frac{1}{e^\epsilon-1}\sum_{i=1}^M c_i^2.
\end{equation}
\begin{definition}[Uniform-Surrogate Privacy-Aware Objective]\label{def:sstq_loss}
For $\mu_{\mathrm{unif}}(dy)=dy/(2B)$, define
\begin{equation}\label{eq:LSSTQ}
\cL_{\mathrm{SSTQ}}(\Gamma)
:=\cL^*_{\mathrm{SSTQ}}(\Gamma;\mu_{\mathrm{unif}})
=\sum_{k=1}^{M-1}\frac{(c_{k+1}-c_k)^3}{12B}
 +\frac{1}{e^\epsilon-1}\sum_{i=1}^M c_i^2,
\end{equation}
subject to $c_1=-B$, $c_M=B$, $c_1\le\cdots\le c_M$, and $\sum_{i=1}^M c_i=0$. Equalities between adjacent entries are permitted because the entries are indexed labels.
\end{definition}

The privacy term is an explicit $\ell_2$ penalty on all labeled codeword positions, including multiplicities. Since the boundary values are fixed, optimization moves the interior entries toward the origin while balancing the interpolation penalty.

\begin{theorem}[Strong Convexity of $\cL_{\mathrm{SSTQ}}$ on the Feasible Set]\label{thm:sstq_convexity}
For $B>0$, $\epsilon>0$, and $M\ge2$, the feasible set is nonempty, compact, and convex. The objective in~\eqref{eq:LSSTQ} is $2/(e^\epsilon-1)$-strongly convex on this set and therefore has a unique global minimizer $\Gamma^*$.
\end{theorem}
The proof is given in Appendix~\ref{app:LSSTQ}.

\noindent\textbf{Improved Worst-Case Bound via Surrogate Codebook Optimization.}
While Theorem~\ref{thm:miracle} gives a worst-case bound with quadratic dependence on $M=2^b$, optimizing the uniform-surrogate objective improves the codebook-dependent part of this upper bound to linear dependence on $M$.

\begin{theorem}\label{thm:one_third_bound}
 Let $\Gamma^*$ minimize $\cL_{\mathrm{SSTQ}}$ over the feasible indexed codebooks. Suppose $B=KR/\sqrt N$ and write $\rho=N/d$. Assuming $M> (4(e^{\epsilon}-1))^{1/3}+1$, the Flat-RR SSTQ with $\Gamma^*$ satisfies
\begin{equation}
\operatorname{MSE}_{\mathrm{SSTQ}}
=\cO\!\left(
\frac{dK^2R^2}{\rho}
\left[1+\frac{2^b}{\epsilon\wedge\epsilon^2}\right]
\right).
\end{equation}
In particular, the bound is linear in $d$ when $K$, $\rho$, $b$, and $R$ are dimension-independent constants.
\end{theorem}

The proof proceeds by constructing a family of candidate codebooks. For an integer parameter \(1 \le m \le \lfloor (M-1)/2 \rfloor\), we define \( \tilde{\Gamma}_m \in \mathcal{D} \) to span exactly \(m\) contiguous gap intervals of size \(B/m\), mapping the corresponding values from \(-B\) to \(B\), while collapsing the remaining \(M - 2m - 1\) indices symmetrically to \(0\) (repeated codewords). We then compare the surrogate loss evaluated at the optimized codebook with the minimum loss over this family, which yields a tighter dependence on \(M\). For comparison, a uniform grid obeys
\[
\sum_{i=1}^M c_i^2
=B^2\frac{M(M+1)}{3(M-1)}=\Theta(B^2M).
\]
At fixed $\epsilon$, its privacy contribution scales as $\Theta_\epsilon(M^2)$ after the Flat-RR multiplier, whereas Theorem~\ref{thm:one_third_bound} gives an $\cO_\epsilon(M)$ upper bound for the optimized surrogate.
\section{Codebook Bit-Width Scalability: The Metric-Aware Laplace Mechanism}
\label{sec:discrete_laplace}
As discussed in Remark~\ref{rem:curse}, standard Flat Randomized Response incurs a variance penalty that grows rapidly with $b$, since the probability gap $p-q$ shrinks on the order of $2^{-b}$. To remove this alphabet-size multiplier from the privacy-noise term, we introduce a \emph{Metric-Aware Laplace Mechanism}. Given an interpolated label $V=k$ with scalar value $c_k$, the mechanism operates in two steps:
\begin{enumerate}
    \item \textbf{Continuous Laplace sampling.} Draw a continuous sample $T$ from a truncated Laplace distribution on $[-B, B]$ centered at $c_k$:
    \begin{equation}\label{eq:laplace}
        f_{T\mid V=k}(t) = \frac{1}{Z_k}\exp\!\left(-\frac{\epsilon|t - c_k|}{2D_B}\right), \quad t \in [-B, B],
    \end{equation}
    where $D_B=2B$ is the domain width and $Z_k=\int_{-B}^{B}\exp(-\epsilon|t-c_k|/(2D_B))dt$ is the normalizing constant.
    \item \textbf{Nearest-codeword quantization.} Fix the deterministic measurable tie rule
    \[
       \tau(t):=\min\arg\min_{i\in[M]}|t-c_i|,
    \]
    release the label $Z=\tau(T)$, and decode it as $c_Z$.
\end{enumerate}
Let $\operatorname{cell}_i:=\{t\in[-B,B]:\tau(t)=i\}$. These cells form a disjoint measurable partition, including when codeword values coincide. The effective transition probabilities are
\begin{equation}\label{eq:laplace-eff}
    P_{i\mid k}=\Pr(Z=i\mid V=k)
    =\frac{1}{Z_k}\int_{\operatorname{cell}_i}
      \exp\!\left(-\frac{\epsilon|t-c_k|}{2D_B}\right)dt.
\end{equation}
\begin{theorem}[Pure $\epsilon$-LDP of Metric-Aware Laplace]
\label{thm:intrinsic_ldp}
For any codebook bounded in $[-B, B]$, the Metric-Aware Laplace Mechanism defined above satisfies pure $\epsilon$-LDP.
\end{theorem}

The metric-aware variant relaxes the requirement of coordinate-level unbiasedness during server-side decoding. Upon receiving the \( \epsilon \)-LDP randomized codeword $c_Z$, the central aggregator applies \emph{identity decoding}, directly setting $\widetilde{y}_J = c_Z$ and projecting it into a 1-sparse vector $\widetilde{y} = N \widetilde{y}_J e_J = Nc_Z e_J$, where we recall $J$ as the (uniformly) sampled index by the encoder. This design explicitly trades a localized bias for reduced variance, enabling seamless integration with arbitrarily optimized, non-uniform continuous codebooks. We refer to Algorithms~\ref{alg:encode-ma}--\ref{alg:decode-ma}
 for the formal description of the encoder and the decoder.

\begin{algorithm}[t]
\caption{SSTQ Client Encoding Algorithm (Metric-Aware)}\label{alg:encode-ma}
\begin{algorithmic}[1]
\REQUIRE $x\in B_2^d(R)$, public scaled Parseval frame $U\in\R^{N\times d}$, and public indexed codebook $\Gamma=(c_1,\ldots,c_M)$ with $M=2^b$ and maximum gap $\Delta_{\max}=\max_k(c_{k+1}-c_k)\le 2BC/(2^b-1)$.
\STATE \textbf{Redundant Kashin Bounding:} Execute Lyubarskii's algorithm~\citep{lyubarskii2010uncertainty} to compute $y \in \R^N$ such that $\frac{d}{N} U^\top y = x$ and $\norminf{y} \le B = \frac{KR}{\sqrt{N}}$.
\STATE \textbf{Oblivious Subsampling:} Sample $J\sim\operatorname{Unif}([N])$, independently of $x$.
\STATE \textbf{Stochastic Interpolation on labels:} If $y_J=c_i$ for one or more indices, set
    \[V=\min\{i:c_i=y_J\}.\]
    Otherwise set $k_-=\max\{i:c_i<y_J\}$ and $k_+=\min\{i:c_i>y_J\}$, and sample $V\in\{k_-,k_+\}$ according to
    \[
    \Pr(V=k_+\mid y_J)=\frac{y_J-c_{k_-}}{c_{k_+}-c_{k_-}},\qquad
    \Pr(V=k_-\mid y_J)=\frac{c_{k_+}-y_J}{c_{k_+}-c_{k_-}}.
    \]
\STATE \textbf{Truncated Laplace perturbation:} Given $V=k$, draw a continuous sample $T\in[-B,B]$ from
    \[
    f_{T\mid V=k}(t)=\frac{1}{Z_k}\exp\!\left(-\frac{\epsilon|t-c_k|}{2D_B}\right),
    \qquad
    D_B=2B,\quad Z_k=\int_{-B}^{B}\exp\!\left(-\frac{\epsilon|t-c_k|}{2D_B}\right)dt.
    \]
\STATE \textbf{Nearest-codeword quantization:} Release the label
    \[Z=\tau(T),\qquad \tau(t):=\min\arg\min_{i\in[M]}|t-c_i|,\]
    where the tie rule makes $\tau$ deterministic and measurable, so that the cells $\operatorname{cell}_i=\{t\in[-B,B]:\tau(t)=i\}$ partition $[-B,B]$ even when codeword values coincide.
\RETURN Transmit $(J,Z)$. \textbf{Cost:} $\lceil\log_2N\rceil+b$ bits.
\end{algorithmic}
\end{algorithm}
\begin{algorithm}[h]
\caption{SSTQ Server Decoding Algorithm (Metric-Aware, Identity Decoding)}\label{alg:decode-ma}
\begin{algorithmic}[1]
\REQUIRE Message $(J,Z)$, indexed codebook $\Gamma=(c_1,\ldots,c_M)$, and scaled Parseval frame $U$.
\STATE \textbf{Identity decoding:} Set $\widetilde y_J=c_Z$ (no $1/(p-q)$ unbiasing factor is applied).
\STATE \textbf{Sparse reconstruction:} Set $\widetilde y=N\widetilde y_Je_J=Nc_Ze_J\in\R^N$.
\STATE \textbf{Frame projection:} Set
    \[\hat x=\frac dN U^\top\widetilde y=d\,c_Z\,u_J.\]
\RETURN $\hat x$.
\end{algorithmic}
\end{algorithm}


\begin{theorem}[MSE and Bias of metric-aware SSTQ]\label{thm:SSTQ-MA}
 We assume a codebook of size $M = 2^b$ with its maximum codebook gap satisfying $\Delta_{\max} = \max_k (c_{k+1} - c_k) \le 2BC/(2^b - 1)$, with $B= KR/\sqrt{N}$ for a structural constant $C \ge 1$. Define the row-coherence parameter $\mu_U:= \max_{j\in[N]} \|u_j\|_2^2$. Then,
\begin{align*}
\MSEMA
&:=\mathbb{E}\!\left[\norm{\widehat{x}-x}^2\right]\\
&\le
\frac{dK^2R^2}{\rho}
\left(
1+\frac{2C^2}{(2^b-1)^2}
+\frac{16(\mu_U+1/d)}{\epsilon}
+\frac{256}{\epsilon^2}
\right)-R^2,\\
\zeta^2
&:=\left\|\mathbb{E}[\widehat{x}]-x\right\|_2^2
\le \frac{64K^2R^2}{\rho\epsilon^2},
\end{align*}
where $\rho=N/d>1$. For any scaled Parseval frame (Definition~\ref{def:frame}), $\mu_U\le\rho$.
\end{theorem}
We refer to Appendix~\ref{app:SSTQ-analysis} for the proof of Theorem~\ref{thm:SSTQ-MA}. 
\medskip

\noindent{\bf Structural Implications.} The theorem gives a dimension-independent and bit-width-independent
upper bound on the bias.
Therefore, the benefit of metric-aware mechanism is that by trading such bias, we obtain a smaller variance, addressing the $O(2^b)$ dependence in the Flat-RR variant. The term $\mathcal{O}(4^{-b})$ in the MSE bound captures the deterministic spatial quantization error. The terms $\mathcal{O}(\epsilon^{-1})$ and $\mathcal{O}(\epsilon^{-2})$ capture the stochastic noise injected to satisfy Local Differential Privacy. 
Together, they define a strict algebraic floor that cannot be mitigated by increasing the grid density.
\medskip

\noindent{\bf Optimal Codebook Bit-Width Allocation.} 
%
Balancing the spatial term with the privacy-dependent terms gives, up to constants,
\[
\frac{C^2}{(2^b-1)^2}\asymp\frac{1}{\epsilon\wedge\epsilon^2},
\qquad
2^b-1\asymp C\sqrt{\epsilon\wedge\epsilon^2}.
\]
Thus a convenient integer scaling is
\[
b_{\mathrm{sat}}
=\max\!\left\{1,\left\lceil
\log_2\!\left(1+\Theta\!\left(C\sqrt{\epsilon\wedge\epsilon^2}\right)\right)
\right\rceil\right\}.
\]
For $\epsilon\le1$, one bit is already sufficient at the level of this bound; for $\epsilon\gg1$, the threshold grows as $\frac12\log_2\epsilon$ up to additive constants. 
\subsection{Codebook Optimization}
Let $X$ denote a random population input.  The frame $U$ is either fixed and public, or publicly randomized independently of $X$; a measurable Kashin-selection rule produces $y(X,U)$, and $J\sim\operatorname{Unif}[N]$ is independent.  We write $Y=y_J(X,U)$.  This distinction is essential: for fixed $(x,U)$, the law induced only by $J$ is the atomic measure $N^{-1}\sum_j\delta_{y_j}$ and need not have a density.  Any density $f_Y$ below is therefore the population law over $(X,U,J)$, not a density generated by frame randomness for one private vector.

The uniform-surrogate objective in Definition~\ref{def:sstq_loss} uses only the public interval $[-B,B]$, where $B=KR/\sqrt N$.  Optimization under the true $f_Y$ is instead an oracle/public-distribution result.  A codebook estimated from private client data would itself be data dependent and requires a separate privacy analysis (or a genuinely public auxiliary sample).

Similar to the flat randomization mechanism, we can write a data-dependent objective loss that captures the codebook-dependent part of MSE bound. For the metric-aware mechanism, the loss is derived in Appendix~\ref{app:MA-loss} and is given by  
$$ \mathcal{L}^*_{\text{MA-SSTQ}}(\Gamma) = \sum_{k=1}^{M-1} \int_{c_k}^{c_{k+1}} W_k (y - c_k)(c_{k+1} - y) f_Y(y) \,dy \;\;+\;\; \sum_{k=1}^M \pi_k \Phi_k\,, $$
where $\pi_k = \int_{-B}^B \Pr[V=k \mid Y=y] f_Y(y) \,dy$ defines the assignment probability of the interpolated label $V$. In addition, $\mu_k = \mathbb{E}[c_Z \mid V = k]$ (the expected output codeword) and $\Phi_k = \mathbb{E}[(c_Z - c_k)^2 \mid V = k]$ (the metric-aware variance), both computed under the effective transition probabilities~\eqref{eq:laplace-eff}, and 
$ W_k = \frac{2(\mu_{k+1} - \mu_k)}{c_{k+1} - c_k} - 1$.

In the flat randomization mechanism, the objective function includes a rigid regularization penalty derived from the variance of uniform point selection: $G(\Gamma) \propto \sum_{i=1}^M c_i^2$. For a uniform grid over $[-B, B]$, this sum diverges as $\mathcal{O}(M)$. Consequently, the uniform grid yields a theoretically vacuous MSE bound as $M$ grows. By optimizing the codebook against a uniform surrogate distribution $f_Y(y) = \frac{1}{2B}$, the optimizer is forced to cluster points near the origin, suppressing the $\mathcal{O}(M)$ penalty to an $\mathcal{O}(1)$ bound. This structural adaptation provides a strict theoretical improvement over the uniform grid, independent of the actual target distribution.

In contrast, identity decoding localizes the metric mechanism's variance: $\sum_k\pi_k \Phi_k\le\max_k \Phi_k\approx \mathcal{O}(\epsilon^{-2})$, independently of the codebook geometry and size. Distribution-aware optimization can improve the spatial term, as formalized below.



\noindent{\bf Optimization under the True Distribution.} We formalize this improvement by introducing the Panter-Dite functional, which quantifies the compressibility of a continuous density. Let the distribution shape factor be defined as:
\[ C_{f} = \frac{1}{4B^2} \left( \int_{-B}^B f_Y(y)^{1/3} \,dy \right)^3 \]
By Hölder's inequality, $C_{f} \le 1$ for all valid densities supported on $[-B, B]$, with equality holding if and only if $f_Y(y)$ is the uniform distribution.



\begin{theorem}\label{thm:loss-MA}
Let $X$ be a random vector with $\norm{X} \le R$ almost surely, and $U$ be a fixed and public Parseval frame or publicly randomized independently of $X$. Define the row-coherence parameter $\mu_U:= \max_{j\in[N]} \|u_j\|_2^2$. Let $Y=y_J(X,U)$ have a known, strictly positive continuous density $f_Y$ on $[-B,B]$, with $B=KR/\sqrt N$, and put $\rho=N/d$. For $M=2^b\ge2$, minimize $\mathcal L^*_{\mathrm{MA-SSTQ}}$ over the compact set $\{-B=c_1\le\cdots\le c_M=B\}$, and let $\Gamma^*$ be a minimizer. Then the population bias and MSE obey
\begin{align*}
\zeta^2_{\rm pop} &:= \|\mathbb{E}[\hat X- X]\|_2^2 \le \frac{64 K^2R^2}{\rho \epsilon^2}, \\
\mathbb{E}[\|\hat X-X\|_2^2]
&\le \frac{d K^2R^2}{\rho}\!\left(1+\frac{2\mu_UC_0^2}{(2^b-1)^2}
+\frac{16(\mu_U+1/d)}{\epsilon}+\frac{256\mu_U}{\epsilon^2}\right)-R^2,
\end{align*}
where $C_0^2\le\min\{1,\frac23C_fR_M\}$.
\end{theorem}
Here $R_M=(M-1)^{-1}\sum_k\sup_{I_k}f_Y/\inf_{I_k}f_Y\ge1$ is the average local oscillation ratio of the density across the intervals $I_k =[c_k,c_{k+1}]$. Since $C_f < 1$ for any strictly non-uniform density, and $R_{M} \to 1$ as $b \to \infty$, we have $C_0^2 < 1$, for sufficiently high resolutions. Also, for standard isotropic random frames, such as a Gaussian matrix followed by Parseval normalization, the row norms concentrate sharply around their common average, and so $\mu_U$ is typically close to one. This explains strict  improvement in the MSE bound compared to the bound derived in Theorem~\ref{thm:SSTQ-MA} (Recall that there  the structural constant is $C\ge1$).
\medskip

\noindent{\bf The Asymptotic Limit ($\epsilon \to \infty$) and the Panter-Dite Framework.} The benefits of codebook optimization formulation become most apparent when considering the limit as privacy constraints relax ($\epsilon \to \infty$). In this regime, the continuous Laplace density $f_{T|v}(t)$ concentrates at a Dirac delta centered at $v$, so the localized variance vanishes ($\Phi_k \to 0$) and the spatial scale modifier $W_k \to 1$. Note that the initial stochastic interpolation step remains active even in this limit; a non-codeword
input is still rounded randomly to its two endpoints followed by deterministic identity output, which is the classical dithered quantizer.

Substituting these limits, the loss simplifies to the classical spatial quantization distortion:
\[ \lim_{\epsilon \to \infty} \cL_{\text{MA-SSTQ}}(\Gamma) = \sum_{k=1}^{M-1} \int_{c_k}^{c_{k+1}} (y-c_k)(c_{k+1}-y) f_Y(y) \,dy \]

In the high-resolution limit ($M \to \infty$) of classical quantization theory, the discrete codebook intervals can be approximated by a continuous point density function $\lambda(y)$. The spatial sum converges to the continuous integral $\frac{1}{6(M-1)^2} \int_{-B}^B \frac{f_Y(y)}{\lambda(y)^2} \,dy$. 

The Panter-Dite framework (1951)~\cite{panter2006quantization} establishes that minimizing this spatial integral via Hölder's inequality strictly requires the continuous point density to satisfy:
\[ \lambda^*(y) = \frac{f_Y(y)^{1/3}}{\int_{-B}^B f_Y(z)^{1/3} \,dz} \]

Therefore, as local differential privacy constraints vanish, the optimal metric-aware codebook abandons variance regularization and allocates discrete coordinates exactly proportional to $f_Y(y)^{1/3}$. This formally bridges optimal differentially private randomization with the classical foundations of scalar source coding.


\section{Federated Learning Convergence Analysis}


We next consider the application of our proposed SSTQ algorithm to private federated learning. We adopt a standard federated learning setup and operate under common assumptions of non-convex optimization and bounded statistical variance, stated below. Consider $n$ clients, with $f_i(w)$ the local loss objective for client $i$, and $F(w) = \frac{1}{n} \sum_{i=1}^n f_i(w)$ the global objective.
\begin{assumption}[Optimization Conditions]
\label{ass:optimization}
Let $\cF_t = \sigma\!\bigl(w_0,\; \hat{g}_{i,s} : i \in [n],\, s < t\bigr)$ denote the
$\sigma$-algebra generated by the initial point and all transmitted gradients
prior to step $t$.  Since $w_t$ is a deterministic function of
$(w_0, \hat{g}_{i,0}, \ldots, \hat{g}_{i,t-1})_{i=1}^{n}$, the iterate $w_t$
is $\cF_t$-measurable.  Write $\E_t[\cdot] = \E[\cdot \mid \cF_t]$.
We assume:
\begin{enumerate}[label=(\roman*)]
    \item \textbf{Objective Smoothness.}
    The global objective $F(w) $ is $L$-smooth: $\norm{\nabla F(x) - \nabla F(y)} \le L\norm{x - y}$ for all $x, y$, and is bounded below by $F^*$. Denote $\Delta_0 \coloneqq F(w_0) - F^*$.
    \item \textbf{Stochastic gradient bounds.}
    Let $g_{i,t} = \nabla f_i(w_t, \xi_{i,t})$ be the stochastic minibatch gradient
    sampled by client~$i$ at step~$t$.  Then, \(\norm{g_{i,t}} \le G\) almost surely and
    \(\E_t\!\bigl[\norm{g_{i,t} - \nabla f_i(w_t)}^2\bigr] \le \sigma^2\).
    \item \textbf{Compression noise.}
    Let $\hat{g}_{i,t}$ be the (private quantized)  gradient transmitted by
    client~$i$.  We assume \(\forall\, i,\, t\), 
        \(\E_t\!\bigl[\norm{\hat{g}_{i,t} - g_{i,t}}^2\bigr] \le \omega\, G^2\).
    \item \textbf{Bounded bias.}
    Define the global server estimator $\hat{g}_t = \frac{1}{n}\sum_{i=1}^{n}\hat{g}_{i,t}$
    and the bias $b_t \coloneqq \E_t[\hat{g}_t] - \nabla F(w_t)$.
    We assume \(
        \norm{b_t}^2 \le \zeta^2\), almost surely, for all $t$.
    \item \textbf{Conditional independence.}
    Conditioned on $\cF_t$, the client estimator errors
    $\{\hat{g}_{i,t} - \nabla f_i(w_t)\}_{i=1}^{n}$ are mutually independent.
\end{enumerate}
\end{assumption}

%
Our next theorem is a general convergence result for distributed SGD under
Assumption~\ref{ass:optimization}.
\begin{theorem}[Non-Convex DSGD Convergence Rate]
\label{thm:convergence}
Under Assumption~\ref{ass:optimization}, consider the iterates
$w_{t+1} = w_t - \eta\,\hat{g}_t$ run for $T$ steps across $n$ clients.
Define $V \coloneqq \frac{2(\omega G^2 + \sigma^2)}{n} + \zeta^2$.
\medskip

\noindent\textbf{(Non-asymptotic bound.)}
Setting $\eta = \min\!\bigl\{\frac{1}{4L},\; \sqrt{\frac{\Delta_0}{L V T}}\bigr\}$
guarantees:
\begin{equation}\label{eq:nonasymptotic}
    \min_{0 \le t \le T-1} \E\!\bigl[\norm{\nabla F(w_t)}^2\bigr]
    \;\le\;
    \frac{16 L \Delta_0}{T} + \frac{8\sqrt{\Delta_0 L V}}{\sqrt{T}} + 2\zeta^2.
\end{equation}
\medskip

\noindent\textbf{(Asymptotic rate.)}
For $T \ge 16 L \Delta_0 / V$, the above simplifies to:
\begin{equation}\label{eq:asymptotic}
    \min_{0 \le t \le T-1} \E\!\bigl[\norm{\nabla F(w_t)}^2\bigr]
    \;\le\;
    \cO\!\left(\frac{\sqrt{\omega G^2 + \sigma^2}}{\sqrt{nT}} + \zeta^2\right),
\end{equation}
where $\cO(\cdot)$ hides dependence on the problem constants $\Delta_0$ and $L$.
\end{theorem}
The formal proof, which utilizes the $L$-smooth descent lemma and controls the bias term via Young's inequality, is deferred to Appendix~\ref{app:fed}.
We next specialize this general convergence theorem to the two variants of our SSTQ framework. 
\begin{corollary}[Convergence under Flat Randomized Response]
\label{cor:convergence_flat}
 SSTQ with Flat RR is unbiased ($\zeta = 0$) by Theorem \ref{thm:sstq_ldp_unbiased} and $\omega=\cO\!\left(\frac{dK^2}{\rho}[1+2^b/(\epsilon\wedge\epsilon^2)]\right)$ under the condition in Theorem~\ref{thm:one_third_bound}. Hence
\begin{equation}
    \min_{0 \le t \le T-1} \E[\norm{\nabla F(w_t)}^2] \le \cO\!\left(\frac{1}{T}+\frac{\sqrt{\frac{dK^2}{\rho}[1+2^b/(\epsilon\wedge\epsilon^2)]G^2+\sigma^2}}{\sqrt{nT}}\right).
\end{equation}
\end{corollary}
For higher codebook bit-widths where the $\cO(2^b)$ penalty becomes a computational bottleneck, we switch the framework to the Metric-Aware Laplace mechanism with identity decoding. This trades absolute unbiasedness for smaller variance bound.

\begin{corollary}[Convergence under Metric-Aware Mechanism]
\label{cor:convergence_metric}
Under the max-gap condition and
fixed row-coherence constant stated in 
Theorem~\ref{thm:SSTQ-MA}, the metric-aware mechanism has $\omega=\cO((dK^2/\rho)(1+\epsilon^{-1}+\epsilon^{-2}))$, and the bias bound $\zeta\le 8KG/(\sqrt\rho\,\epsilon)$. Hence
\begin{equation}
    \min_{0 \le t \le T-1} \E[\norm{\nabla F(w_t)}^2] \le \cO\!\left(\frac{1}{T}+\frac{\sqrt{\frac{dK^2}{\rho}G^2(1+\epsilon^{-1}+\epsilon^{-2})+\sigma^2}}{\sqrt{nT}}+\frac{K^2G^2}{\rho\epsilon^2}\right).
\end{equation}
\end{corollary}

\section{Neural Network Experiments}\label{sec:nn-experiments}

We evaluate the practical performance of the LDP mechanisms in a federated learning setting with neural network models on two standard image classification benchmarks: Fashion-MNIST and CIFAR-10.

\paragraph{Experimental Setup.}
We train a fully connected neural network with a single hidden layer of 32~units and ReLU activations.
For Fashion-MNIST (10 classes, $784$-dimensional inputs), the model has $d = 25{,}450$ parameters; for CIFAR-10 (10 classes, $3{,}072$-dimensional inputs), the model has $d = 98{,}666$ parameters.
For all Kashin-based methods (SQKR, SSTQ), we construct the frame using a randomized partial Hadamard transform of size $N = 2^{\lceil\log_2(2.5\,d)\rceil}$.
We simulate a federated setting with $W = 100$ workers, each holding a disjoint partition of the training data.
At each of the $T = 100$ communication rounds, every worker computes a stochastic gradient on a mini-batch of size~$64$, clips it to $\|g\|_2 \le C = 0.1$, applies the LDP mechanism, and transmits the privatized message to the server.
The server aggregates the received messages, applies a gradient norm cap of~$10.0$ for stability, and updates the global model with a constant learning rate ($\eta = 0.2$ for Fashion-MNIST, $\eta = 0.5$ for CIFAR-10).
All methods operate under pure $\epsilon$-LDP with per-round $\epsilon = 3$. SSTQ uses bit-width $b = 4$.

\begin{remark}[Effect of server-side clipping]
Server-side gradient clipping is a nonlinear operation: $\E[{\rm clip}(\hat{g})] \ne {\rm clip}(\E[\hat{g}])$ in general, so clipping introduces a small bias even when the underlying estimator is unbiased. In the convergence framework of Theorem~\ref{thm:convergence}, this clipping-induced bias is absorbed into the $\zeta^2$ term. In our experiments, the client-side clipping at $C = 0.1$ combined with the frame scaling $B = KR/\sqrt{N}$ ensures that the reconstructed one-coordinate estimator has bounded norm, and the server-side cap of $10.0$ activates infrequently. The theoretical convergence rates in Corollaries~\ref{cor:convergence_flat} and~\ref{cor:convergence_metric} hold exactly for the unclipped algorithm; the clipped variant converges to a neighborhood that includes the additional clipping bias.
\end{remark}

\begin{remark}[Per-round vs.\ cumulative privacy budget]\label{rem:composition}
All reported privacy parameters are \emph{per-round} guarantees: each client's message in each round satisfies $\epsilon$-LDP independently. Under basic composition with $T$~rounds of full participation, the cumulative budget is $\epsilon_{\mathrm{total}} = T\epsilon$. We report per-round $\epsilon$ (i.e., message-level privacy) because it isolates the mechanism's intrinsic privacy--utility trade-off and is the standard metric across all baselines compared (vqSGD, SQKR, PrivUnit). In practice, client subsampling (where each client participates in only a fraction of rounds) provides privacy amplification by subsampling~\cite{balle2018privacy}, and advanced composition~\cite{kairouz2015composition} yields $\epsilon_{\mathrm{total}} = \mathcal{O}(\epsilon\sqrt{T\log(1/\delta)})$ under $(\epsilon,\delta)$-DP. These orthogonal techniques apply equally to all mechanisms and are independent of the quantization scheme.
\end{remark}

To ensure a fair comparison, we use a common random seed for mini-batch selection across all methods, so that each method trains on exactly the same sequence of data batches.
The only source of randomness that differs across methods is the LDP noise.
Test accuracy is evaluated on the full held-out test set ($10{,}000$ images) at each round.

\paragraph{Results.}
Figure~\ref{fig:nn-experiments} presents the training loss, test accuracy, and accuracy-vs-communication-cost tradeoff for both datasets.
The communication cost per client per round is summarized in Table~\ref{tab:bits}.

\begin{figure}[ht]
  \centering
  \includegraphics[width=\textwidth]{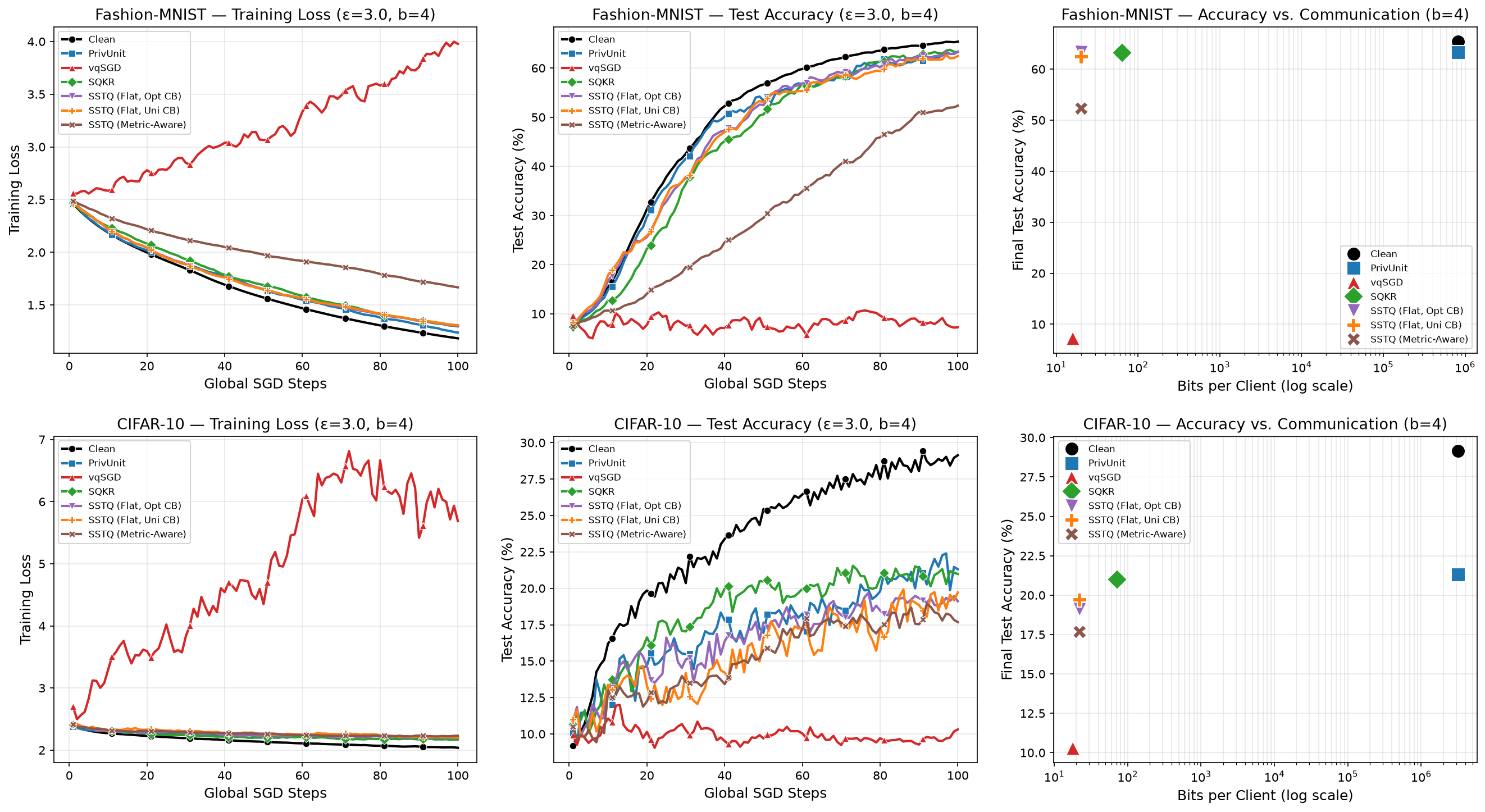}
  \caption{Neural network experiments under $\epsilon$-LDP ($\epsilon=3$, $b=4$, $W=100$ workers).
  \textbf{Top row:} Fashion-MNIST ($d=25{,}450$). \textbf{Bottom row:} CIFAR-10 ($d=98{,}666$).
  \textbf{Left:} Training cross-entropy loss vs.\ SGD steps.
  \textbf{Center:} Test accuracy vs.\ SGD steps.
  \textbf{Right:} Final test accuracy vs.\ bits transmitted per client (log scale).}
  \label{fig:nn-experiments}
\end{figure}

\begin{table}[ht]
  \centering
  \caption{Communication cost per client per round ($b=4$, $\epsilon=3$).}
  \label{tab:bits}
  \begin{tabular}{lccc}
    \toprule
    \textbf{Method} & \textbf{Formula} & \textbf{Fashion-MNIST} & \textbf{CIFAR-10} \\
    \midrule
    Clean / PrivUnit & $32d$ & $814{,}400$ & $3{,}157{,}312$ \\
    vqSGD & $\lceil\log_2 d\rceil+1$ & $16$ & $18$ \\
    SQKR & $k(\lceil\log_2 N\rceil + 1)$ & $68$ & $76$ \\
    SSTQ (all variants) & $\lceil\log_2 N\rceil + b$ & $20$ & $22$ \\
    \bottomrule
  \end{tabular}
\end{table}

\emph{Fashion-MNIST.}
On the training loss (Figure~\ref{fig:nn-experiments}, top-left), Clean SGD and PrivUnit decrease steadily, reaching final losses of approximately $1.2$ and $1.25$, respectively.
SQKR and SSTQ (Flat-RR) converge at a slightly slower rate, settling around $1.3$, which reflects the additional variance from quantization and privacy noise.
SSTQ (Metric-Aware) converges more slowly (final loss $\approx 1.7$), consistent with the small bias $\zeta^2$ from the identity decoder (Corollary~\ref{cor:convergence_metric}).
vqSGD exhibits \emph{divergent} training loss, increasing from $2.5$ to $\approx 4.0$, confirming that its $O(d^3/\epsilon^2)$ variance scaling drastically degrades its performance at  $d = 25{,}450$.

On test accuracy (Figure~\ref{fig:nn-experiments}, top-center), Clean SGD reaches $65.4\%$ and PrivUnit closely tracks at $63.3\%$.
Among quantized methods, SQKR achieves $63.2\%$ using $68$~bits per client, while SSTQ (Flat-RR, optimized codebook) achieves $63.3\%$ and SSTQ (Flat-RR, uniform codebook) reaches $62.4\%$---both using only $20$~bits per client, a $3.2\times$ reduction compared to SQKR.
Notably, the SSTQ Flat variants and SQKR now achieve nearly the same accuracy as PrivUnit, with the gap between all three quantized methods being within $1\%$.
vqSGD fails to learn ($7.2\%$, near random chance).

We note that SSTQ (Metric-Aware) underperforms SSTQ (Flat-RR) due to its bias. However, its advantage becomes apparent at larger bit widths $(b)$, as discussed in Appendix~\ref{app:add-exp} and demonstrated in Figure~\ref{fig:sgd-b4-vs-b8} therein.

\emph{CIFAR-10.}
The higher dimensionality ($d = 98{,}666$) amplifies the variance effects.
On training loss (Figure~\ref{fig:nn-experiments}, bottom-left), Clean SGD and PrivUnit steadily decrease, while SQKR and SSTQ variants remain relatively flat around $2.2$.
vqSGD again diverges dramatically, with training loss rising to $\approx 5.7$.
On test accuracy (Figure~\ref{fig:nn-experiments}, bottom-center), Clean SGD reaches $29.1\%$ and PrivUnit $21.3\%$.
Among quantized methods, SQKR achieves $21.0\%$ (76 bits), followed by SSTQ (Flat-RR, uniform codebook) at $19.7\%$, SSTQ (Flat-RR, optimized codebook) at $19.1\%$ (both using 22 bits), and SSTQ (Metric-Aware) at $17.7\%$.
vqSGD reaches $10.3\%$ (18 bits).

\paragraph{Accuracy vs.\ Communication Cost.}
The rightmost panels of Figure~\ref{fig:nn-experiments} present the accuracy-vs-bits tradeoff (see also Table~\ref{tab:bits}).
SSTQ achieves accuracy comparable to SQKR while requiring substantially fewer bits per client.
On Fashion-MNIST, SSTQ (Flat-RR, optimized codebook) achieves $63.3\%$---essentially matching SQKR's $63.2\%$---while using $3.4\times$ fewer bits ($20$ vs.\ $68$ bits).
On CIFAR-10, the gap is similarly small: SSTQ (Flat-RR, optimized codebook) reaches $19.1\%$ compared to SQKR's $21.0\%$, a difference of only $2\%$, while using $3.45\times$ fewer bits ($22$ vs.\ $76$ bits).
This communication advantage arises because SQKR transmits $k=\min\{\lceil\epsilon\log_2e\rceil,b_0\}$ coordinates, each requiring $\lceil\log_2 N\rceil + 1$ bits to encode, whereas SSTQ transmits a single coordinate index plus a $b$-bit quantization level, yielding a total of $\lceil\log_2 N\rceil + b$ bits.

vqSGD uses the fewest bits ($16$--$18$) but fails to converge, confirming that its cubic variance scaling renders it impractical in high dimensions.
PrivUnit achieves the best accuracy among private methods but requires full-precision transmission ($>$800K bits), making it impractical for communication-constrained settings. It is worth noting that the $32d$ cost reported in Table~\ref{tab:bits} reflects float32 encoding. Using lower precision (e.g., $16d$ or $8d$) reduces the constant but does not change the $\Theta(d)$ scaling, which is inherent to any mechanism that outputs a $d$-dimensional vector.

\paragraph{Computational Cost.}
All Kashin-based methods (SQKR, SSTQ) share the same computational bottleneck: the iterative Kashin representation, which involves repeated applications of the randomized partial Hadamard transform on a vector of dimension $N$, costing $O(N \log N)$ per client per round.
In our experiments, each SGD step takes approximately $18$\,s on Fashion-MNIST and $65$\,s on CIFAR-10 (single CPU), with the Kashin transform accounting for over $99\%$ of the per-worker computation.
The remaining operations (quantization, RR or Laplace sampling) are $O(M)$ for SSTQ and $O(N)$ for SQKR. The server-side reconstruction requires applying a single column of the frame transform, costing $O(d)$ per client, which is negligible compared to the Kashin encoding.

In contrast, vqSGD avoids the Kashin representation entirely: it operates directly on the gradient vector using simplex sampling, costing only $O(d)$ per client.
This makes vqSGD significantly faster per step (approximately $5$--$10\times$ compared to Kashin-based methods), but as demonstrated above, this computational advantage is negated by its prohibitive variance scaling in high dimensions.

PrivUnit also avoids Kashin representations ($O(d)$ per client) but requires transmitting full-dimensional vectors, making it communication-efficient only in a per-dimension sense.

\bigskip

\textbf{Conclusion and Limitations.}  SSTQ is a principled framework for private quantization. By combining discrete quantization with adaptive, privacy-aware codebook optimization, it avoids the variance penalties associated with fixed-geometry schemes. Its current formulation assumes a fixed communication budget and known data range; future work could address nonstationary streams and develop robust, range-agnostic adaptive mechanisms for more flexible real-world deployment.
\bigskip

{\bf Acknowledgment.} Some of the analytical derivations were first obtained using the long proof harness in Google Antigravity 2.0. The authors verified the entire paper and take full responsibility for all the results.

\bibliographystyle{ieeetr}
\bibliography{ref}

\newpage
\appendix
\addcontentsline{toc}{section}{Appendices}

\section{Curse of Dimensionality in Fixed-Geometry LDP}\label{app:curse}
\label{sec:curse}

Previous studies in geometric vector quantization, most notably vqSGD \citep{gandikota2021vqsgd}, have claimed that quantizing vectors to a cross-polytope and applying Randomized Response yields a variance of \(\mathcal{O}(d^2/\epsilon^2)\) for the proposed unbiased estimator. We challenge this assertion and demonstrate that this formulation actually incurs a variance of \(\Theta(d^3/\epsilon^2)\). This fundamental scaling issue critically impairs the performance of high-dimensional gradient embeddings in privacy-constrained settings.
\medskip

\noindent\textbf{Formalization of vqSGD and Cross-Polytope LDP.}
Geometric vector quantization methods, such as vqSGD, typically map continuous vectors $x \in \Sph^{d-1}$ to the vertices of a discrete geometric object. A canonical choice is the cross-polytope codebook, defined as:
\begin{equation}
    \mathcal{C}_{cp} = \{ \pm \sqrt{d} e_i \}_{i=1}^d,
\end{equation}
where $e_i$ is the $i$-th standard basis vector in $\R^d$. Note that the total size of this discrete domain is exactly $M = 2d$ vertices.

To achieve pure $\epsilon$-Local Differential Privacy, vqSGD applies a Flat Randomized Response (RR) to a selected vertex $v \in \mathcal{C}_{cp}$. The mechanism outputs a noisy vertex $z \in \mathcal{C}_{cp}$ such that the transition probabilities are:
\begin{align}
    p &= \Pr(z = v \mid v) = \frac{e^\epsilon}{e^\epsilon + 2d - 1}, \label{eq:p} \\
    q &= \Pr(z = c \neq v \mid v) = \frac{1}{e^\epsilon + 2d - 1}. \label{eq:q}
\end{align}
Using the fact that the sum of vertices is zero, the server debiases the output $z$ using a simple scaling factor of $\frac{1}{p-q}$ and outputs the estimator $\hat{x} = \frac{z}{p-q}$.
\medskip

\noindent\textbf{The High-Dimensional Variance Growth.}
Many analyses conceptually treat the unbiasing scalar $\frac{1}{(p-q)^2}$ as an $\cO(1/\epsilon^2)$ constant when evaluating the expected Mean Squared Error (MSE). We explicitly show that because the codebook size $M = 2d$ is intertwined with the ambient dimension, the probability gap $p-q$ shrinks as $\cO(1/d)$.

\begin{theorem}[Exact cross-polytope RR exhibits the dimensionality curse]
\label{thm:curse_proof}
For any input vector $x \in \Sph^{d-1}$ quantized via the Cross-Polytope $\mathcal{C}_{cp}$ and perturbed via Flat Randomized Response for pure $\epsilon$-LDP, the exact Mean Squared Error (variance) of the unbiased estimator scales asymptotically as:
\begin{equation}
    \text{MSE}_{{\rm vqSGD}} = d-1 + \frac{4d^2}{e^{\epsilon}-1} + \frac{4d^3}{(e^{\epsilon}-1)^2}.
\end{equation}
Consequently it is $\Theta_\epsilon(d^3)$ for fixed $\epsilon>0$, and $\Theta(d^3/\epsilon^2)$ jointly when $0<\epsilon\le1$.  If $\epsilon$ grows faster than $\log d$, the error instead approaches the nonprivate stochastic-quantization rate of $\Theta(d)$.
\end{theorem}
Before proceeding to the proof, we provide a remark on alternative polytopes, noting that they are subject to analogous limitations.

\begin{remark}[Scope beyond the cross-polytope]
The same cubic conclusion follows for a family only after verifying three ingredients: $M=\Theta(d)$ labels, codeword squared norm $\Theta(d)$, and a centered RR decoder whose debiasing gap is $\Theta(1/d)$ at fixed privacy. Conversely, mapping to dense uniform hypercubes bypasses this geometric penalty but necessitates transmitting $\mathcal{O}(d)$ bits per client, which fails to satisfy the fundamental requirements of federated communication compression.
\end{remark}

%
\begin{proof}[Proof of Theorem~\ref{thm:curse_proof}]
Let $x \in \Sph^{d-1}$ be an input unit vector. Assume the client stochastically quantizes $x$ to a true vertex $v$ in the cross-polytope codebook $\mathcal{C}_{cp} = \{ \pm \sqrt{d} e_i \}_{i=1}^d$ such that $\E[v \mid x] = x$. The pure $\epsilon$-LDP Flat Randomized Response mechanism then outputs a noisy token $z \in \mathcal{C}_{cp}$ given $v$. The server constructs the unbiased estimator $\hat{x} = \frac{z}{p-q}$. 

To compute the Mean Square Error (MSE), we use the unbiasedness property of the estimator to write $\E[\hat{x}] = \E_v[ \E_{z \mid v}[\frac{z}{p-q} \mid v] \mid x] = \E[v] = x$, and therefore:
\begin{align}
    \text{MSE}_{\text{vqSGD}} &= \E[\norm{\hat{x}- x}^2] \nonumber \\
    &= \E[\norm{\hat{x}}^2] - 2\E[\langle \hat{x}, x \rangle] + \norm{x}^2 \nonumber \\
    &= \E[\norm{\hat{x}}^2] - 2\langle x, x \rangle + \norm{x}^2 \nonumber \\
    &= \E[\norm{\hat{x}}^2] - \norm{x}^2 \nonumber \\
    &= \E\left[\norm{\frac{z}{p-q}}^2\right] - 1 \nonumber \\
    &= \frac{1}{(p-q)^2} \E[\norm{z}^2] - 1. \label{eq:mse_expansion}
\end{align}

By definition of the cross-polytope codebook $\mathcal{C}_{cp}$, every  vertex $z \in \mathcal{C}_{cp}$ takes the form $\pm \sqrt{d} e_i$, and so $\norm{z}^2= d$.

Next, we evaluate the unbiasing scalar. The cross-polytope $\mathcal{C}_{cp}$ consists of exactly $M = 2d$ discrete vertices. Under Flat Randomized Response for pure $\epsilon$-LDP over $M$ elements, the transition probabilities are defined as:
\begin{equation}
    p = \frac{e^\epsilon}{e^\epsilon + 2d - 1}, \quad \text{and} \quad q = \frac{1}{e^\epsilon + 2d - 1}.
\end{equation}
The probability gap evaluates to:
\begin{equation}
    p - q = \frac{e^\epsilon - 1}{e^\epsilon + 2d - 1}.
\end{equation}
Substituting this multiplier and $\E[\norm{z}^2] = d$ directly back into Equation~\eqref{eq:mse_expansion}:
\begin{equation}
    \text{MSE}_{\text{vqSGD}} = \frac{d}{(p-q)^2} - 1 = d \left( \frac{e^\epsilon + 2d - 1}{e^\epsilon - 1} \right)^2 - 1.
\end{equation}

We expand the  numerator by writing $(e^\epsilon + 2d - 1)^2$ as $((e^\epsilon - 1) + 2d)^2$:
\begin{align*}
    \text{MSE}_{\text{vqSGD}} &= \frac{d}{(e^\epsilon - 1)^2} \Big( (e^\epsilon - 1)^2 + 4d(e^\epsilon - 1) + 4d^2 \Big) - 1\\
 &= d \left( 1 + \frac{4d}{e^\epsilon - 1} + \frac{4d^2}{(e^\epsilon - 1)^2} \right) - 1 \nonumber \\
    &= d - 1 + \frac{4d^2}{e^\epsilon - 1} + \frac{4d^3}{(e^\epsilon - 1)^2}.
\end{align*}

For fixed $\epsilon>0$, the cubic term dominates as $d\to\infty$, giving $\Theta_\epsilon(d^3)$.  Uniformly for $0<\epsilon\le1$, $e^\epsilon-1=\Theta(\epsilon)$ and the expression is $\Theta(d^3/\epsilon^2)$.
 This completes the proof.
\end{proof}

\section{Proofs for the SSTQ Mechanism}\label{app:SSTQ-1}
This appendix provides the formal mathematical proofs for the theoretical guarantees of the Subsampled Stochastic TurboQuant (SSTQ) architecture with Flat randomization response introduced in Section~\ref{sec:SSTQ-FR}. We first prove that SSTQ  satisfies pure $\epsilon$-Local Differential Privacy and guarantees an unbiased gradient estimator. We then establish a bound on the conditional variance of the 1D randomized response mechanism, by which we prove that SSTQ achieves the information-theoretic optimal $\cO(d/\epsilon^2)$ Mean Squared Error.

\subsection{Proof of Strict $\epsilon$-LDP and Unbiasedness}

\begin{proof}[Proof of Theorem \ref{thm:sstq_ldp_unbiased}]
We establish the privacy and utility guarantees in two distinct parts.

\textbf{Part 1: Strict pure $\epsilon$-LDP.}
The SSTQ client encoding algorithm transmits a pair $(J, Z)$, where $J\in [N]$ is the subsampled coordinate index and $Z \in [M]$ is a randomized-reponse label. 
First, the subsampling index $J$ is generated uniformly at random from $[N]$, totally independent of the client's private data $x$. Because the distribution of $J$ does not depend on $x$, its transmission consumes exactly $\epsilon = 0$ privacy budget. 

Second, the stochastic quantization step maps the continuous Kashin coordinate $y_J$ into a finite, discrete codebook $\Gamma = \{c_1, \dots, c_M\}$ bounded in $[-B, B]$, yielding a codeword label $V \in [M]$. This step is a local data-independent mapping applied to the already bounded coordinate.

Third, the algorithm applies Flat Randomized Response over the constant domain $M = 2^b$ to produce the codeword label $Z\in[M]$. By the definition of the mechanism, the transition probabilities are $p = \frac{e^\epsilon}{e^\epsilon + M - 1}$ if $Z = V$, and $q = \frac{1}{e^\epsilon + M - 1}$ if $Z \neq V$. For any two possible labels $i, i' \in [M]$ and any label $\ell \in [M]$, the maximum probability ratio is strictly bounded by:
\begin{equation}
    \frac{\Pr(Z=\ell \mid V = i)}{\Pr(Z=\ell \mid V=i')} \le \frac{p}{q} = \frac{\frac{e^\epsilon}{e^\epsilon + M - 1}}{\frac{1}{e^\epsilon + M - 1}} = e^\epsilon.
\end{equation}
Because the maximum probability ratio between any two inputs is exactly $e^\epsilon$, the mechanism satisfies pure $\epsilon$-LDP  on the client device, without requiring a trusted shuffler or shared public coin.

\textbf{Part 2: Unbiasedness.}
We trace the expectations backwards from the server's estimator $\hat{x}$ to the true vector $x$. 
First, the server receives $(J,Z)$ and constructs $\widetilde{y}_J = \frac{c_Z}{p-q}$. Given the  zero-mean codebook constraint ($\sum_{i=1}^M c_i = 0$), the expected value of $c_Z$ given the interpolated label $V$ is:
\begin{equation}
    \E[c_Z \mid V=i] = p c_i + q \sum_{\ell \neq i} c_\ell = p c_i + q \left( \sum_{\ell=1}^M c_\ell - c_i \right) = p c_i + q (0 - c_i) = (p - q)c_i.
\end{equation}
Thus, $\E[\widetilde{y}_J \mid V=i] = \E\left[\frac{c_Z}{p-q} \;\middle|\; V=i\right] = c_i$.

Second, the label $V$ is generated via exact linear stochastic interpolation of the coordinate $y_J \in [c_k, c_{k+1}]$. The expected value is:
\begin{equation}
    \E[c_V \mid y_J] = c_{k+1} \Pr(V = k+1) + c_k \Pr(V = k) = c_{k+1} \left( \frac{y_J - c_k}{c_{k+1} - c_k} \right) + c_k \left( \frac{c_{k+1} - y_J}{c_{k+1} - c_k} \right) = y_J.
\end{equation}

Third, taking the expectation over the uniform oblivious subsampling index $J\in[N]$, the server's reconstructed 1-sparse vector $\widetilde{y} = N\widetilde{y}_J  e_J$ yields:
\begin{equation}
    \E_J[\widetilde{y}] = \sum_{j=1}^N \Pr(J=j) (N y_j  e_j) = \sum_{j=1}^N \frac{1}{N} (N \cdot y_j \cdot e_j) = \sum_{j=1}^N y_j e_j = y.
\end{equation}

Finally, the server projects $\widetilde{y}$ through the scaled-Parseval frame $U$. Because the Kashin representation  guarantees exact reconstruction $\frac{d}{N} U^T y = x$ (Theorem \ref{thm:kashin_rep}), applying linearity of expectation yields:
\begin{equation}
    \E[\hat{x}] = \E\left[\frac{d}{N} U^T \widetilde{y}\right] = \frac{d}{N} U^T \E[\widetilde{y}] = \frac{d}{N} U^T y = x.
\end{equation}
The global estimator is therefore strictly unbiased.
\end{proof}

\subsection{MSE bound for SSTQ (Flat-RR)}

We begin by proving Lemma~\ref{lem:trace} on the frame trace property.
\begin{proof}[Proof of Lemma \ref{lem:trace}]
Consider the trace of the matrix $U U^\top \in \R^{N \times N}$. By the cyclic permutation property of the trace operator:
\begin{equation}
    \text{Tr}(U U^\top) = \text{Tr}(U^\top U).
\end{equation}
We are given $U^\top U = \frac{N}{d} I_d$. Therefore:
\begin{equation}
    \text{Tr}(U^\top U) = \text{Tr}\left(\frac{N}{d} I_d\right) = \frac{N}{d} \times d = N.
\end{equation}
The diagonal elements of the outer product $U U^\top$ are  the squared Euclidean norms of the rows of $U$. Thus, $\sum_{j=1}^N \norm{u_j}^2 = \text{Tr}(U U^\top) = N$. 
\end{proof}

To formalize how SSTQ addresses the curse of dimensionality, we first bound the conditional variance of the 1D randomized response mechanism operating over the bounded Kashin coefficients.

\begin{lemma}\label{lem:factor_id}
For any stochastically quantized client token bounded by $|c_V| \le B=\frac{KR}{\sqrt{N}}$, the server's debiased scalar satisfies:
\begin{equation}
    \E[\widetilde{y}_J^2 \mid V] \le \frac{B^2}{(p-q)^2}.
\end{equation}
\end{lemma}

Because the discrete domain size $M$ over which the randomized response mechanism operates is now defined by the constant payload bit-width ($M = 2^b$), the unbiasing probability gap $(p-q)$ is structurally decoupled from the high-dimensional geometry. This decoupling enables to bypass the curse of dimensionality.

\begin{proof}[Proof of Lemma \ref{lem:factor_id}]
We bound the uncentered conditional second moment  $\tilde{y}_J = \frac{c_Z}{p-q}$. The second moment of the Flat Randomized Response output $c_Z$ is:
\begin{equation}
    \E[c_Z^2 \mid V=i] = p c_i^2 + q \sum_{\ell \neq i} c_\ell^2 = (p-q)c_i^2 + q \sum_{\ell=1}^M c_\ell^2.
\end{equation}
Dividing by the unbiasing scalar squared $(p-q)^2$ gives:
\begin{equation}
    \E[\widetilde{y}_J^2 \mid V=i] = \E\left[\left(\frac{c_Z}{p-q}\right)^2 \;\middle|\; V=i\right] = \frac{1}{p-q} c_i^2 + \frac{q}{(p-q)^2} \sum_{\ell=1}^M c_\ell^2.
\end{equation}

By $|c_i|, |c_\ell| \le B$, substituting the common bound $B^2$ gives
\begin{align}
    \E[\tilde{y}_J^2 \mid V] &\le \frac{B^2}{p-q} + \frac{q MB^2}{(p-q)^2} = \frac{B^2}{(p-q)^2},
\end{align}
where we used that the randomization probabilities sum to one: $p + (M-1)q = 1$.  
\end{proof}

We next proceed to bound the MSE achieved by SSTQ (Flat-RR variant).
\begin{proof}[Proof of Theorem \ref{thm:miracle}]
We evaluate the expected Mean Squared Error (total variance) of the full $d$-dimensional reconstructed gradient $\hat{x}$. Let $U \in \R^{N \times d}$ be a scaled Parseval frame where $U^T U = \frac{N}{d} I_d$. 

The server reconstructs the target gradient via the projection $\hat{x} = \frac{d}{N} U^T \widetilde{y}$. The 1-sparse vector is $\widetilde{y} = N \widetilde{y}_j e_j$. Thus, the projection evaluates to:
\begin{equation}
    \hat{x} = \frac{d}{N} U^T (N \widetilde{y}_j e_j) = d \widetilde{y}_j U^T e_j = d \widetilde{y}_j u_j,
\end{equation}
where $u_j^T$ is the $j$-th row of the frame matrix $U$.

We calculate the expected squared $\ell_2$ norm of the estimator $\hat{x}$:
\begin{equation}
    \E[\norm{\hat{x}}^2] = \E\left[ \norm{d \widetilde{y}_J u_J}^2 \right] = d^2 \E[\widetilde{y}_J^2 \norm{u_J}^2].
\end{equation}
applying Lemmas~\ref{lem:factor_id} and~\ref{lem:trace},
\begin{equation}\label{eq:factorized_expectation}
\E\|\hat x\|_2^2
=\frac{d^2}{N}\sum_{j=1}^N\|u_j\|_2^2\E[\tilde y_j^2\mid J=j]
\le \frac{d^2B^2}{(p-q)^2}
=\frac{dK^2R^2}{\rho(p-q)^2}.
\end{equation}

Because the estimator is unbiased, the MSE is $\E\|\hat x\|_2^2-\|x\|_2^2$.  For $\|x\|_2=R$, Equation~\eqref{eq:factorized_expectation} gives
\begin{equation}\label{eq:mse-final}
\operatorname{MSE}_{\rm SSTQ}\le \frac{dK^2R^2}{\rho}\left(1+\frac{M}{e^\epsilon-1}\right)^2-R^2.
\end{equation}


Finally, we evaluate the unbiasing scalar. Because the algorithm operates over a 1D scalar codebook, the codebook size $M = 2^b$ is an $\cO(1)$ constant completely independent of $d$. The probability gap $(p-q)$ for Flat Randomized Response is:
\begin{equation}
    p - q = \frac{e^\epsilon - 1}{e^\epsilon + 2^b - 1}.
\end{equation}
Therefore, 
\begin{align*}
    \frac{1}{(p-q)^2} = \left(1+\frac{2^b}{e^\epsilon-1}\right)^2 
    <1+ \frac{2^{b+1}}{\epsilon} + \frac{4^b}{\epsilon^2} 
    \le 1+ 4^b\left(\frac{1}{\epsilon}+ \frac{1}{\epsilon^2}\right)\le 1+ \frac{2\cdot 4^b}{\epsilon\wedge\epsilon^2}.
\end{align*}
Using this bound into~\eqref{eq:mse-final}, we get
$$ \text{MSE}_{\text{SSTQ}} = \cO\left(\frac{dK^2R^2}{\rho}\left(1+ \frac{4^b}{\epsilon\wedge \epsilon^2}\right) \right)\,,$$
which completes the proof.

\end{proof}

\section{Proofs for Surrogate Codebook Optimization}\label{app:LSSTQ}
This appendix provides the  analytical derivations and formal proofs for the surrogate continuous codebook optimization framework ($\cL_{SSTQ}$) introduced in Section \ref{sec:opt-code}.

\subsection{Derivation of the Surrogate Minimax Objective}\label{sec:surrogate-derivation}

We start by deriving the exact closed-form relation between the expected Mean Squared Error of the Local Differential Privacy mechanism and the continuous codebook geometry.

\begin{proof}[Derivation of Equation \eqref{eq:expected_mse} and Definition \ref{def:sstq_loss}]
Let $y \in [-B, B]$ be a bounded Kashin coordinate, and assume it lies in the interval $[c_k, c_{k+1}]$ between two adjacent centroids in the codebook $\Gamma$. As defined in Algorithm \ref{alg:encode}, $y$ is stochastically quantized to a codeword with label $V \in \{c_k, c_{k+1}\}$ via exact linear interpolation.
The conditional variance of this stochastic interpolation step is:
\begin{align}
    r_\Gamma(y)&:=\Var(c_V \mid Y=y) = \E[c_V^2 \mid y] - y^2 \nonumber \\
    &= c_{k+1}^2 \left( \frac{y - c_k}{c_{k+1} - c_k} \right) + c_k^2 \left( \frac{c_{k+1} - y}{c_{k+1} - c_k} \right) - y^2\nonumber\\
    &= (y - c_k)(c_{k+1} - y).\label{eq:var-vy}
\end{align}

Next, the Flat Randomized Response mechanism maps the label $V$ to a noisy label $Z \in [M]$. The server computes the debiased estimator $\widetilde{y}_J = \frac{c_Z}{p-q}$. Because Flat RR operates over the $M$ centroids in $\Gamma$, its second moment evaluates to $\E[c_Z^2 \mid V] = (p-q)c_V^2 + q\sum_{\ell=1}^M c_\ell^2$. 
Dividing by $(p-q)^2$, we obtain:
\begin{equation}
    \E[\tilde{y}_J^2 \mid V] = \frac{1}{p-q} c_V^2 + \frac{q}{(p-q)^2} \sum_{\ell=1}^M c_\ell^2.
\end{equation}
Since $\E[\widetilde{y}_J \mid V] = c_V$, the conditional variance is:
\begin{equation}
    \Var(\widetilde{y}_J \mid V) = \E[\widetilde{y}_J^2 \mid V] - c_V^2 = \left(\frac{1}{p-q} - 1\right)c_V^2 + \frac{q}{(p-q)^2}\sum_{\ell=1}^M c_\ell^2.
\end{equation}

By invoking the Law of Total Variance over the two cascaded stochastic mechanisms (quantization followed by LDP), the variance of the estimator given the continuous coordinate $y$ is:
\begin{align}
    \Var(\widetilde{y}_J \mid Y=y) &= \E_V[\Var(\widetilde{y}_J \mid V) \mid Y=y] + \Var_V(\E[\widetilde{y}_J \mid V] \mid Y=y) \nonumber \\
    &= \E_V\left[ \left(\frac{1}{p-q} - 1\right)c_V^2 + \frac{q}{(p-q)^2}\sum_{\ell=1}^M c_\ell^2 \ \middle| \ Y=y \right] + \Var(c_V \mid Y=y).
\end{align}
Substituting $\E_V[c_V^2 \mid Y=y] = \Var(c_V \mid Y=y) + y^2$, we get:
\begin{align}
    \Var(\widetilde{y}_J \mid Y=y) &= \left(\frac{1}{p-q} - 1\right)\big(\Var(c_V \mid Y=y) + y^2\big) + \frac{q}{(p-q)^2}\sum_{\ell=1}^M c_\ell^2 + \Var(c_V \mid Y=y) \nonumber \\
    &= \frac{r_\Gamma(y)}{p-q}  + \left(\frac{1}{p-q} - 1\right)y^2 + \frac{q}{(p-q)^2}\sum_{\ell=1}^M c_\ell^2.
\end{align}

Because the estimator is locally unbiased ($\E[\widetilde{y}_J \mid y] = y$), we have:
\begin{equation}
    \E[\widetilde{y}_J^2 \mid Y=y] = \Var(\widetilde{y}_J \mid Y=y) + y^2 = \frac{r_\Gamma(y)}{p-q} + \frac{1}{p-q} y^2 + \frac{q}{(p-q)^2} \sum_{\ell=1}^M c_\ell^2.
\end{equation}

Taking the unconditional expectation over $Y~\sim \mu$, factoring out $\frac{1}{p-q}$, and applying the Flat RR probability identity $\frac{q}{p-q} = \frac{1}{e^\epsilon - 1}$, we derive the surrogate loss formulation $\cL^*_{{\rm SSTQ}}(\Gamma;\mu)$:
\begin{align}
    \E_Y[\widetilde{y}_J^2] &= \frac{1}{p-q} \left[ \E_Y[r_\Gamma(y)] + \frac{q}{p-q} \sum_{\ell=1}^M c_\ell^2 \right] + \frac{1}{p-q} \E_Y[y^2] \nonumber \\
    &= \frac{1}{p-q} \underbrace{\left[ \E_Y[r_\Gamma(y)] + \frac{1}{e^\epsilon - 1} \sum_{\ell=1}^M c_\ell^2 \right]}_{= \cL^*_{\text{SSTQ}}(\Gamma;\mu)} + \frac{1}{p-q} \E_Y[y^2].
\end{align}
If $\mu$ has a density $f_Y$, by writing the two expectations above as integration against $f_Y$, we obtain the characterization of $\cL^*_{\text{SSTQ}}(\Gamma;\mu)$ given by~\eqref{eq:exact-loss}.
\end{proof}

\subsection{Proof of Global Convexity}

\begin{proof}[Proof of Theorem \ref{thm:sstq_convexity}]
We decompose the objective function into the spatial quantization penalty $F(\Gamma)$ and the variance regularization penalty $G(\Gamma)$:
\[ \mathcal{L}_{\text{SSTQ}}(\Gamma) = F(\Gamma) + G(\Gamma) \]
where
\[ F(\Gamma) = \sum_{k=1}^{M-1} \frac{(c_{k+1}-c_k)^3}{12B}, \quad \text{and} \quad G(\Gamma) = \lambda \sum_{i=1}^M c_i^2 \]
with the privacy-dependent scalar $\lambda = \frac{1}{e^\epsilon - 1}$. For any valid privacy budget $\epsilon > 0$, the parameter $\lambda$ is strictly positive.
We show that $F$ is convex and $G$ is strictly convex and hence the loss is a strictly convex function of $\Gamma$.

Let $D \in \mathbb{R}^{(M-1) \times M}$ be the finite difference matrix defined such that the $k$-th element of the matrix-vector product is $(D\Gamma)_k = c_{k+1} - c_k$. The feasible set $\mathcal{D}$ corresponds to the region where $D\Gamma \ge 0$ element-wise.

We define the scalar function $\phi(x) = \frac{x^3}{12B}$. For $x \ge 0$, its second derivative is $\phi''(x) = \frac{x}{2B} \ge 0$.
The term $F(\Gamma)$ can be written as $F(\Gamma) = \sum_{k=1}^{M-1} \phi((D\Gamma)_k)$. By the chain rule for vectors, the Hessian matrix of $F(\Gamma)$ is given by:
\[ \nabla^2 F(\Gamma) = D^\top \mathrm{diag}\Big(\phi''((D\Gamma)_1), \dots, \phi''((D\Gamma)_{M-1})\Big) D \]
Because $(D\Gamma)_k \ge 0$ for all $\Gamma \in \mathcal{D}$, the diagonal matrix contains solely non-negative entries. Consequently, for any vector $v \in \mathbb{R}^M$, the quadratic form $v^\top \nabla^2 F(\Gamma) v \ge 0$. This establishes that the Hessian $\nabla^2 F(\Gamma)$ is positive semi-definite ($\nabla^2 F(\Gamma) \succeq 0$), and thus $F(\Gamma)$ is convex on $\mathcal{D}$.

The regularization term evaluates to a scaled squared Euclidean norm, $G(\Gamma) = \lambda \|\Gamma\|_2^2$. The Hessian matrix of this canonical quadratic form is:
\[ \nabla^2 G(\Gamma) = 2\lambda I_M \]
where $I_M$ is the $M \times M$ identity matrix. Because $\lambda > 0$, the matrix $2\lambda I_M$ is strictly positive definite ($\nabla^2 G(\Gamma) \succ 0$). This establishes that $G(\Gamma)$ is strongly convex on $\mathbb{R}^M$.

The Hessian of the total objective function is the sum of the constituent Hessians:
\[ \nabla^2 \mathcal{L}_{\text{SSTQ}}(\Gamma) = \nabla^2 F(\Gamma) + \nabla^2 G(\Gamma) \]
The sum of a positive semi-definite matrix and a strictly positive definite matrix is strictly positive definite. Thus, $\nabla^2 \mathcal{L}_{\text{SSTQ}}(\Gamma) \succeq 2\lambda I_M \succ 0$ for all $\Gamma \in \mathcal{D}$. This rigorously guarantees that $\mathcal{L}_{\text{SSTQ}}(\Gamma)$ is strictly globally convex (and formally, strongly convex) over the feasible domain.

We next prove the existence of the global minimizer via  coercivity. A continuous function defined on a closed set is guaranteed to attain a global minimum if it is coercive, meaning the function value diverges to infinity as the norm of the input diverges ($\lim_{\|\Gamma\|_2 \to \infty} \mathcal{L}_{\text{SSTQ}}(\Gamma) = \infty$).

For any $\Gamma \in \mathcal{D}$, the condition $c_{k+1} - c_k \ge 0$ guarantees that $F(\Gamma) \ge 0$. Consequently, the total loss function is strictly bounded from below by its quadratic regularization term:
\[ \mathcal{L}_{\text{SSTQ}}(\Gamma) \ge G(\Gamma) = \lambda \|\Gamma\|_2^2 \]
Taking the limit as the norm of the codebook grows yields:
\[ \lim_{\|\Gamma\|_2 \to \infty} \mathcal{L}_{\text{SSTQ}}(\Gamma) \ge \lim_{\|\Gamma\|_2 \to \infty} \lambda \|\Gamma\|_2^2 = \infty \]
The strong coercivity of the objective function on the closed convex set $\mathcal{D}$ ensures that the infimum is finite and attained. Thus, there exists at least one global minimizer $\Gamma^* \in \mathcal{D}$. 

Note that if explicit boundary constraints such as $c_1 = -B$ and $c_M = B$ are enforced, the monotonic ordering requirement confines all coordinates strictly to the interval $[-B, B]$. The feasible set $\mathcal{D}$ then becomes a compact subset of $\mathbb{R}^M$, and existence is directly guaranteed by the Weierstrass extreme value theorem without requiring the coercivity argument.

The uniqueness of global minimizer follows from the strict convexity (if there are two distinct global minimizers, their average will attain a smaller function value which is a contradiction.)


\end{proof}

\subsection{Improved MSE Bound Using Optimized Codebook}
\begin{proof}[Proof of Theorem \ref{thm:one_third_bound}]
By recalling  Equation~\eqref{eq:factorized_expectation}, we have
\begin{align}\label{eq:MSE-B}
\text{MSE}_{\text{SSTQ}} = \E[\norm{\hat{x}}^2] -\norm{x}^2<\E[\norm{\hat{x}}^2]  =\frac{d^2}{N}\sum_{j=1}^N\|u_j\|_2^2\E[\tilde y_j^2\mid J=j]\,. 
\end{align}
Next by combining~\eqref{eq:exact-loss}, \eqref{eq:expected_mse}, and using Lemma~\ref{lem:trace}, we get
\begin{align}
\text{MSE}_{\text{SSTQ}} \le
\frac{d^2}{p-q} \left[ \sum_{k=1}^{M-1} \int_{c_k}^{c_{k+1}} (y - c_k)(c_{k+1} - y)f_Y(y) \,dy \;\;+\;\; \frac{1}{e^\epsilon - 1} \sum_{i=1}^M c_i^2 + \int_{-B}^B y^2 f_Y(y)\,dy \right]  
\end{align}
We denote the right-hand side (without the scaling factor $\frac{d^2}{p-q}$) by $U(\Gamma)$ as a function of the codebook $\Gamma$, and decompose it as the data-dependent part $I(\Gamma)$ and the regularization part $G(\Gamma)$ given by
\begin{align}
I(\Gamma)&:= \sum_{k=1}^{M-1}\int_{c_k}^{c_{k+1}} [(y - c_k)(c_{k+1} - y) + y^2]f_Y(y) \,dy \\
G(\Gamma) &:= \lambda \sum_{i=1}^M c_i^2\,,
\end{align}
with $\lambda = (e^{\epsilon}-1)^{-1}$, and so $U(\Gamma) = I(\Gamma)+G(\Gamma)$.

We next upper bound $I(\Gamma)$. Define the polynomial within the brackets as $h_k(y)$ for $y \in [c_k, c_{k+1}]$:
$$ h_k(y) = (y - c_k)(c_{k+1} - y) + y^2 $$

Expanding $h_k(y)$ yields:
$$ h_k(y) = y c_{k+1} - y^2 - c_k c_{k+1} + c_k y + y^2 = (c_k + c_{k+1})y - c_k c_{k+1} $$

The function $h_k(y)$ is an affine function with respect to $y$. By the properties of affine functions on closed intervals, $h_k(y)$ attains its maximum at one of the boundaries of the interval $[c_k, c_{k+1}]$. Evaluating $h_k(y)$ at the endpoints gives:
$$ h_k(c_k) = (c_k + c_{k+1})c_k - c_k c_{k+1} = c_k^2 $$
$$ h_k(c_{k+1}) = (c_k + c_{k+1})c_{k+1} - c_k c_{k+1} = c_{k+1}^2 $$

Thus, for all $y \in [c_k, c_{k+1}]$, the function is bounded by:
$$ h_k(y) \le \max(c_k^2, c_{k+1}^2) $$

Because the codebook elements are constrained to the interval $[-B, B]$, it follows that $c_i^2 \le B^2$ for all $i \in \{1, \dots, M\}$. Therefore, $h_k(y) \le B^2$ uniformly for all $k \in \{1, \dots, M-1\}$ and $y \in [c_k, c_{k+1}]$.

Applying this pointwise upper bound to the integrals results in:
$$ I(\Gamma) \le \sum_{k=1}^{M-1} \int_{c_k}^{c_{k+1}} B^2 f_Y(y) \,dy = B^2 \int_{-B}^B f_Y(y) \,dy $$

Because $f_Y(y)$ is a valid probability density function supported on $[-B, B]$, its integral over this domain evaluates to $1$. Hence, $I(\Gamma) \le B^2$. This bound is independent of the distribution $f_Y(y)$ and holds for any feasible codebook, including $\Gamma^*$.

We next upper bound the regularization term $G(\Gamma^*)$. To  constrain $G(\Gamma^*)$, we evaluate $J(\Gamma) = F(\Gamma) + G(\Gamma)$ where $F(\Gamma) = \sum_{k=1}^{M-1} \frac{\Delta_k^3}{12B}$. By the continuous global optimality of $\Gamma^*$, $J(\Gamma^*) \le J(\tilde{\Gamma})$ for any valid proxy codebook sequence.
For an integer parameter $1 \le m \le \lfloor \frac{M-1}{2} \rfloor$, define $\tilde{\Gamma}_m \in \mathcal{D}$ to span exactly $m$ contiguous gap intervals of size $B/m$ mapping respectively from $-B$ and $B$, collapsing the remaining $M - 2m - 1$ indices symmetrically to $0$. 
This sequence computes strictly to $F(\tilde{\Gamma}_m) = 2m \frac{(B/m)^3}{12B} = \frac{B^2}{6m^2}$ and $G(\tilde{\Gamma}_m) = 2 \lambda \sum_{j=1}^m \left( \frac{j B}{m} \right)^2 = \lambda B^2 \frac{(m+1)(2m+1)}{3m}$. Applying Jensen's inequality to the objective yields a strict minimal floor $F(\Gamma^*) \ge \frac{2B^2}{3(M-1)^2}$. Subtracting this extracts the uniform limit:
$$ G(\Gamma^*) \le F(\tilde{\Gamma}_m) + G(\tilde{\Gamma}_m) - F(\Gamma^*) \le \frac{B^2}{6m^2} + \frac{B^2 (m+1)(2m+1)}{3m(e^\epsilon - 1)} - \frac{2B^2}{3(M-1)^2} $$
Combining the bounds on $I(\Gamma^*)$ and $G(\Gamma^*)$ we obtain
\[
U(\Gamma^*) \le B^2+ \min_{1 \le m \le \lfloor \frac{M-1}{2} \rfloor} B^2 \left( \frac{1}{6m^2} + \frac{(m+1)(2m+1)}{3m(e^\epsilon - 1)} \right) - \frac{2B^2}{3(M-1)^2}
\]
To quantify the term with minimum over $m$ on the right-hand side, consider the continuous relaxation of the objective function, defined approximately as $g(m) \approx B^2 \left( \frac{1}{6m^2} + \frac{2 \lambda m}{3} \right)$. Setting the first derivative to zero provides the unconstrained optimum $m^* \approx (2\lambda)^{-1/3}$. We proceed by setting $m = (2\lambda)^{-1/3} \vee 1$, since we should have $m\ge 1$. In addition, by our assumption, $M> (4(e^{\epsilon}-1))^{1/3}+1$, which implies that $\frac{M-1}{2}\ge m$, and so this is a valid choice.

Substituting $m$ back into the objective yields the asymptotic behavior:
$$ \frac{1}{6m^2} + \frac{(m+1)(2m+1)}{3m(e^\epsilon - 1)}  = \mathcal{O}(\lambda^{2/3} \vee \lambda) =\cO(\epsilon^{-1}\vee \epsilon^{-2/3}) \,,$$
using the fact that $\lambda<1/\epsilon$. This also implies that $\frac{2B^2}{3(M-1)^2} = \cO(\lambda^{2/3}B^2) =\cO(B^2 \epsilon^{-2/3})$. Hence,
\[U(\Gamma^*) = \cO(B^2 (1+ \epsilon^{-1}+\epsilon^{-2/3}))=\cO(B^2(1+\epsilon^{-1})).\]
Consequently, 
\[
\text{MSE}_{\text{SSTQ}} = \cO\left(\frac{d^2}{p-q} B^2 (1+\epsilon^{-1}))\right) = \cO\left(\frac{dK^2R^2}{\rho}(1+M\epsilon^{-1}+ M\epsilon^{-2})\right)\,, 
\]
given that $B= KR/\sqrt{N}$, $\rho= N/d$, $p-q = 1/(1+M\lambda)$, and so $1/(p-q) < 1+ M\epsilon^{-1}$. This completes the proof.
\end{proof}

\section{Proofs for the Metric-Aware Laplace Mechanism}\label{app:metric_privacy}
This appendix provides the formal mathematical proofs for the Metric-Aware Laplace mechanism introduced in Section \ref{sec:discrete_laplace}. The mechanism operates in two stages: (1) continuous Laplace sampling, and (2) nearest-codeword quantization. We first prove that the mechanism satisfies pure $\epsilon$-LDP via the post-processing theorem. We then formulate the surrogate objective $\cL_{\text{SSTQ}}^{\text{MA}}$ and establish upper bounds on the variance and the bias of the estimator under identity decoding. The global bias is entirely independent of the ambient dimension.

\subsection{Proof of $\epsilon$-LDP}

\begin{proof}[Proof of Theorem \ref{thm:intrinsic_ldp}]
The proof follows from the post-processing theorem of differential privacy. The mechanism consists of two stages:

\textbf{Stage 1 (Continuous Laplace Sampling).} The Kashin coefficient  $y_J$ is mapped to a codeword $c_V$ via stochastic rounding. The mechanism draws $T$ from a truncated Laplace distribution on $[-B, B]$ with density
$f_{T|V=k}(t) = \frac{1}{Z_k}\exp\left(-\frac{\epsilon|t - c_k|}{2D_B}\right)$ for $t \in [-B, B]$,
where $D_B = 2B$. This continuous mechanism satisfies $\epsilon$-LDP: for any two codewords $c_k, c_m \in \Gamma$ and any measurable set $S \subseteq [-B, B]$,
\begin{align}
\frac{\Pr(T \in S \mid V = k)}{\Pr(T \in S \mid V = m)} = \frac{\int_S \frac{1}{Z_k} e^{-\frac{\epsilon|t-c_k|}{2D_B}} \, dt}{\int_S \frac{1}{Z_m} e^{-\frac{\epsilon|t-c_m|}{2D_B}} \, dt}.
\end{align}
For the numerator, using the triangle inequality $|t - c_k| \ge |t - c_m| - |c_m - c_k|$:
\begin{align}
e^{-\frac{\epsilon|t-c_k|}{2D_B}} \le e^{\frac{\epsilon|c_m-c_k|}{2D_B}} \cdot e^{-\frac{\epsilon|t-c_m|}{2D_B}}.
\end{align}
Integrating over $S$ yields $\int_S \frac{1}{Z_k} e^{-\frac{\epsilon|t-c_k|}{2D_B}} dt \le \frac{Z_m}{Z_k} e^{\frac{\epsilon|c_m-c_k|}{2D_B}} \int_S \frac{1}{Z_m} e^{-\frac{\epsilon|t-c_m|}{2D_B}} dt$.

An analogous bound shows $\frac{Z_m}{Z_k} \le e^{\frac{\epsilon|c_m-c_k|}{2D_B}}$. Multiplying yields:
\begin{equation}
\frac{\Pr(T \in S \mid V = k)}{\Pr(T \in S \mid V = m)} \le \exp\left(\frac{\epsilon|c_k - c_m|}{D_B}\right) \le e^\epsilon,
\end{equation}
where the final step uses $|c_k - c_m| \le 2B = D_B$.

\textbf{Stage 2 (Nearest-Codeword Quantization).} The output $Z = \tau(T)$ with $\tau(t):=\min\arg\min_{i \in [M]} |t - c_i|$, is a deterministic function of $T$. By the post-processing theorem of differential privacy, applying a deterministic function to the output of an $\epsilon$-LDP mechanism preserves $\epsilon$-LDP.

In addition, since the subsampled index $J$ is chosen independently of the input $x$, its transmission does not add any privacy loss. Therefore, the composed mechanism satisfies pure $\epsilon$-LDP for any codebook geometry.
\end{proof}

\subsection{MSE and Bias Bounds under Identity Decoding}\label{app:SSTQ-analysis}
\subsubsection{Derivation of Surrogate Loss Objective for the Metric-Aware Mechanism}\label{app:MA-loss}
To optimize the codebook $\Gamma = \{c_1, \dots, c_M\}$ under the Metric-Aware Laplace mechanism, the loss objective must evaluate the expected mean squared error. We define 
\begin{align}\label{eq:aux-MS-SSTQ}
\mathcal{L}^*_{\text{MA-SSTQ}}(\Gamma) := \mathbb{E}[(y - c_Z)^2],
\end{align}
utilizing identity decoding.

We first recall the pipeline. Let $y$ denote the continuous coordinate drawn from the density $f_Y(y)$. The mechanism pipeline operates as follows:
\begin{enumerate}
\item {\bf Stochastic Quantization:} The Kashin coordinate $y \in [c_k, c_{k+1}]$ is stochastically assigned to a codeword $c_V$  with label $V \in \{k, k+1\}$ to enforce the unbiasedness property $\mathbb{E}[c_V|Y=y] = y$. The transition probabilities are $p_k(y) = \frac{c_{k+1} - y}{c_{k+1} - c_k}$ and $p_{k+1}(y) = \frac{y - c_k}{c_{k+1} - c_k}$.
\item  {\bf Metric-Aware Privatization:} Given the label $V = j$, the mechanism draws a continuous sample $T$ from the truncated Laplace distribution $f_{T|V=j}(t) = \frac{1}{Z_j}\exp\left(-\frac{\epsilon|t - c_j|}{4B}\right)$ for $t \in [-B, B]$, and outputs the nearest codeword $c_Z$, where $Z = \tau(T)$ with $\tau(t):=\min\arg\min_{i \in [M]} |t - c_i|$. The effective transition probability $P_{i|j} = \Pr[Z=i \mid V=j]$ is given by Eq.~\ref{eq:laplace-eff}.
\item  {\bf Identity Decoding:} The decoder assigns the reconstructed coordinate directly as $\widetilde{y} = c_Z$.
\end{enumerate}
By the Law of Total Expectation, the conditional error decomposes as:
$$ \mathbb{E}[(y - c_Z)^2 \mid Y=y] = \mathbb{E}_{V|Y}\left[ \mathbb{E}_{Z|V}[ (y - c_Z)^2 ] \right] $$

For a fixed intermediate codeword label $V = j$, we define two geometric statistics representing the mechanism's behavior:
\begin{itemize}
\item {\bf Expected Structural Bias:} $\mu_j = \mathbb{E}[c_Z \mid V=j]$, the expected output codeword under the effective transition probabilities~\eqref{eq:laplace-eff}.

\item {\bf Localized Mechanism Variance:} $\Phi_j = \mathbb{E}[(c_Z - c_j)^2 \mid V=j]$.
\end{itemize}
Expanding the inner expectation explicitly around $c_j$ gives:
$$ \mathbb{E}_{Z|V}[ (y - c_j + c_j - c_Z)^2 ] = (y - c_j)^2 + 2(y - c_j)(c_j - \mu_j) + \Phi_j $$

Evaluating the outer expectation over the label $V \in \{k, k+1\}$ strictly partitions the polynomial into three distinct components:
\begin{enumerate}
\item Base Quantization Error:
Applying the probabilities yields the standard error polynomial:
$$ p_k(y)(y - c_k)^2 + p_{k+1}(y)(y - c_{k+1})^2 = \frac{c_{k+1}-y}{c_{k+1}-c_k}(y-c_k)^2 + \frac{y-c_k}{c_{k+1}-c_k}(y-c_{k+1})^2 $$
Factoring out the term $(y - c_k)(c_{k+1} - y)$ leaves:
$$ \frac{(y - c_k)(c_{k+1} - y)}{c_{k+1} - c_k} \left[ (y - c_k) + (c_{k+1} - y) \right] = (y - c_k)(c_{k+1} - y) $$

\item Cross-Correlation Bias Term:
The expected cross-term captures the interaction between the quantization interval bounds and the mechanism's structural bias:
$$ 2 p_k(y)(y - c_k)(c_k - \mu_k) + 2 p_{k+1}(y)(y - c_{k+1})(c_{k+1} - \mu_{k+1}) $$
Applying the sign substitution $y - c_{k+1} = -(c_{k+1} - y)$ and replacing $p_k(y), p_{k+1}(y)$ transforms this to:
$$ 2 \frac{c_{k+1} - y}{c_{k+1} - c_k} (y - c_k)(c_k - \mu_k) - 2 \frac{y - c_k}{c_{k+1} - c_k} (c_{k+1} - y)(c_{k+1} - \mu_{k+1}) $$
Factoring out the data-dependent polynomial provides the algebraic simplification:
$$ 2 \frac{(y - c_k)(c_{k+1} - y)}{c_{k+1} - c_k} \left[ c_k - \mu_k - (c_{k+1} - \mu_{k+1}) \right] = 2(y - c_k)(c_{k+1} - y) \left( \frac{\mu_{k+1} - \mu_k}{c_{k+1} - c_k} - 1 \right) $$

\item Expected Privacy Variance:
The expected structural variance interpolates linearly between the two valid codeword states:
$$ \mathbb{E}_{V|Y}[\Phi_V] = p_k(y) \Phi_k + p_{k+1}(y) \Phi_{k+1} $$
\end{enumerate}

Summing the base quantization error and the cross-correlation bias exactly consolidates the $y$-dependent portion of the loss. We define the scale modifier $W_k$:
$$ W_k = 1 + 2\left( \frac{\mu_{k+1} - \mu_k}{c_{k+1} - c_k} - 1 \right) = \frac{2(\mu_{k+1} - \mu_k)}{c_{k+1} - c_k} - 1 $$

Integrating the components over the density $f_Y(y)$ establishes the intermediate objective:
$$ \mathcal{L}^*_{\text{MA-SSTQ}}(\Gamma) = \sum_{k=1}^{M-1} \int_{c_k}^{c_{k+1}} \left[ W_k (y - c_k)(c_{k+1} - y) + p_k(y) \Phi_k + p_{k+1}(y) \Phi_{k+1} \right] f_Y(y) \,dy $$

By regrouping the local variance assignments globally across the entire support domain $[-B, B]$, the continuous expected mean squared error minimizes to its definitive analytical form:

$$ \mathcal{L}^*_{\text{MA-SSTQ}}(\Gamma) = \sum_{k=1}^{M-1} \int_{c_k}^{c_{k+1}} W_k (y - c_k)(c_{k+1} - y) f_Y(y) \,dy \;\;+\;\; \sum_{k=1}^M \pi_k \Phi_k\,, $$
where $\pi_k = \int_{-B}^B \Pr[V=k \mid Y=y] f_Y(y) \,dy$ defines the  assignment probability of the interpolated label $V$.

\subsubsection{Proof of Theorem \ref{thm:SSTQ-MA}}

Let \(u_j^\top\) denote the \(j\)-th row of \(U\), and set
\[
w_j:=\|u_j\|_2^2,
\qquad
D_U:=\operatorname{diag}(w_1,\ldots,w_N),\qquad
 \mu_U:=\max_{j\in[N]}w_j.
\]
Since \(U\) is a scaled Parseval frame,
\(
U^\top U=\frac{N}{d}I_d=\rho I_d
\)
, and the trace identity (Lemma~\ref{lem:trace}) gives
\(
\sum_{j=1}^Nw_j=\operatorname{Tr}(U^\top U)=N.\)

Let \(y\in\mathbb R^N\) be a Kashin representation of \(x\), so that
\[
x=\frac{d}{N}U^\top y,
\qquad
\|y\|_\infty\le B=\frac{KR}{\sqrt N}.
\]
For each coordinate \(y_j\), let \(c_{V,j}\) denote the output of stochastic
endpoint interpolation and let \(c_{Z,j}\) denote the scalar obtained after
applying the metric-aware mechanism and identity decoding. Here, we made the dependence on $j$ explicit in the index. Thus, if
\(y_j\in[c_k,c_{k+1}]\), then
\[
\mathbb E[c_{V,j}\mid y_j]=y_j
\]
and
\[
\mathbb E[(c_{V,j}-y_j)^2\mid y_j]
=
(y_j-c_k)(c_{k+1}-y_j)
\le \frac{\Delta_{\max}^2}{4}.
\]

The scalar metric-aware mechanism satisfies the uniform moment bounds
\[
\max_i
\left|
\mathbb E[c_Z-c_i\mid V=i]
\right|
\le \frac{8B}{\epsilon},
\qquad
\max_i
\mathbb E[(c_Z-c_i)^2\mid V=i]
\le \frac{128B^2}{\epsilon^2}.
\]
Define
\[
r_j:=\mathbb E[c_{Z,j}\mid y_j]-y_j,
\qquad
a_j:=\mathbb E[(c_{Z,j}-y_j)^2\mid y_j].
\]
Because stochastic endpoint interpolation is unbiased, \(r_j\) is a
convex combination of the conditional mechanism biases at the two
neighboring codewords. Hence
\[
|r_j|\le \frac{8B}{\epsilon}.
\]
Moreover, using
\[
c_{Z,j}-y_j=(c_{V,j}-y_j)+(c_{Z,j}-c_{V,j})
\]
and the inequality \((s+t)^2\le 2s^2+2t^2\), we obtain
\[
\begin{aligned}
a_j
&\le
2\mathbb E[(c_{V,j}-y_j)^2\mid y_j]
+
2\mathbb E[(c_{Z,j}-c_{V,j})^2\mid y_j] \\
&\le
\frac{\Delta_{\max}^2}{2}
+
\frac{256B^2}{\epsilon^2}\\
&\le
\frac{2B^2C^2}{(2^b-1)^2}
+
\frac{256B^2}{\epsilon^2}.
\end{aligned}
\]
Denote the last upper bound by
\[
A:=
\frac{2B^2C^2}{(2^b-1)^2}
+
\frac{256B^2}{\epsilon^2}.
\]

The SSTQ estimator can be written as
\(
\widehat x=dc_{Z,J}u_J\), with
\(J\sim\operatorname{Unif}[N].
\)

We first bound its bias. Writing
\(
r:=(r_1,\ldots,r_N)^\top,
\)
the unbiased reconstruction identity
\(x=(d/N)U^\top y\) gives
\[
\mathbb E[\widehat x]-x
=
\frac{d}{N}U^\top r.
\]
Since
\[
\|UU^\top\|_{\mathrm{op}}
=
\|U^\top U\|_{\mathrm{op}}
=
\frac Nd,
\]
we have
\[
\left\|\mathbb E[\widehat x]-x\right\|_2^2
=
\frac{d^2}{N^2}\|U^\top r\|_2^2\\
\le
\frac{d}{N}\|r\|_2^2\\
\le
d\left(\frac{8B}{\epsilon}\right)^2\\
=
\frac{64K^2R^2}{\rho\epsilon^2}.
\]
This proves the claimed bias bound.

We next consider the mean-squared error. Define
\(
P:=\frac{d}{N}UU^\top.
\)
The scaled Parseval identity implies that \(P\) is an orthogonal
projection. Furthermore,
\[
Ux=Py,
\qquad
\|x\|_2^2=\frac{d}{N}y^\top Py.
\]
Since
\[
\mathbb E[c_{Z,j}^2\mid y_j]
=
a_j+y_j^2+2y_jr_j,
\]
a direct expansion gives
\[
\begin{aligned}
\mathbb E\!\left[\|\widehat x-x\|_2^2\right]
={}&
\frac{d^2}{N}\sum_{j=1}^N w_j a_j
+
\frac{d^2}{N}y^\top D_Uy
-\|x\|_2^2\\
&\quad+
\frac{2d}{N}y^\top(dD_U-P)r.
\end{aligned}
\]
For the first term, \(w_j\le\mu_U\) and \(a_j\le A\) imply
\[
\frac1N\sum_{j=1}^N w_j a_j
\le
A\frac1N\sum_{j=1}^N w_j
\le
A.
\]
Therefore,
\[
\frac{d^2}{N}\sum_{j=1}^N w_j a_j
\le
d^2B^2
\left(
\frac{2C^2}{(2^b-1)^2}
+
\frac{256}{\epsilon^2}
\right).
\]
For the second and the third terms, 
using $\|y\|_{\infty}\le \frac{K\norm{x}}{\sqrt{N}}$ and 
\(\sum_jw_j=N\), we have
\[
y^\top D_Uy
=
\sum_{j=1}^Nw_jy_j^2
\le
\|y\|_{\infty}^2\sum_{j=1}^Nw_j
=
N\|y\|_{\infty}^2\le K^2\norm{x}^2.
\]
Thus,
\begin{align}\label{eq:trick-R}
\frac{d^2}{N}y^\top D_Uy -\norm{x}^2\le \left(\frac{K^2d}{\rho}-1\right)\norm{x}^2\le \left(\frac{K^2d}{\rho}-1\right)R^2,
\end{align}
because $dK^2/\rho\ge 1$ and $\norm{x}\le R$.
Finally,
\[
\|y\|_2\le\sqrt N B,
\qquad
\|r\|_2\le\frac{8\sqrt N B}{\epsilon},
\]
and
\[
\|dD_U-P\|_{\mathrm{op}}
\le
d\|D_U\|_{\mathrm{op}}+\|P\|_{\mathrm{op}}
\le
d\mu_U+1.
\]
It follows that
\[
\begin{aligned}
\left|
\frac{2d}{N}y^\top(dD_U-P)r
\right|
&\le
\frac{2d}{N}
\|y\|_2
\|dD_U-P\|_{\mathrm{op}}
\|r\|_2\\
&\le
\frac{16dB^2}{\epsilon}(d\mu_U+1)\\
&=
\frac{16d^2B^2}{\epsilon}
\left(\mu_U+\frac1d\right).
\end{aligned}
\]
Combining the preceding bounds and
\[
d^2B^2
=
\frac{d^2K^2R^2}{N}
=
\frac{dK^2R^2}{\rho},
\]
gives
\[
\mathbb E\!\left[\|\widehat x-x\|_2^2\right]
\le
\frac{dK^2R^2}{\rho}
\left(
1+
\frac{2C^2}{(2^b-1)^2}
+
\frac{16(\mu_U+1/d)}{\epsilon}
+
\frac{256}{\epsilon^2}
\right)
-R^2.
\]

To complete the proof of theorem, it remains to bound \(\mu_U\). For every \(j\),
\[
\begin{aligned}
\|u_j\|_2^2
&=
e_j^\top UU^\top e_j\\
&\le
\|UU^\top\|_{\mathrm{op}}\\
&=
\|U^\top U\|_{\mathrm{op}}
=
\rho.
\end{aligned}
\]
Taking the maximum over \(j\) proves that
\(
\mu_U\le\rho.
\)

\subsubsection{Proof of Theorem~\ref{thm:loss-MA}}
Let \(X\) be the population input, with \(\norm{X}=R\) almost surely.  The
frame \(U\in\R^{N\times d}\) is fixed and public, or publicly randomized
independently of \(X\).  For every realization under consideration,
\[
 U^\top U=\frac Nd I_d=\rho I_d,
\]
where \(u_j^\top\) is the \(j\)-th row of \(U\).  Write
\[
 w_j:=\norm{u_j}^2,
 \qquad
 D_U:=\operatorname{diag}(w_1,\ldots,w_N),
 \qquad
 \mu_U:=\max_{j\in[N]}w_j.
\]
If \(U\) is publicly randomized, assume the same deterministic upper bound
\(\max_jw_j(U)\le\mu_U\) holds almost surely.  The trace identity (Lemma~\ref{lem:trace}) gives
\(
 \sum_{j=1}^Nw_j=\operatorname{Tr}(U^\top U)=N.
 \)

A measurable
Kashin-selection rule produces \(y=y(X,U)\in\R^N\) satisfying
\[
 X=\frac dN U^\top y,
 \qquad
 \lVert y\rVert_\infty\le B:=\frac{KR}{\sqrt N}.
\]
Let \(J\sim\operatorname{Unif}[N]\) be independent of \((X,U)\), and define
\(Y=y_J(X,U)\).  This is the index used by the SSTQ mechanism.

\textbf{1. Conditional and population coordinate laws.}
For fixed \((x,u)\), the conditional law of the selected coordinate is
\begin{align}\label{eq:cond-PU}
 P_{Y\mid X=x,U=u}
 =\frac1N\sum_{j=1}^N\delta_{y_j(x,u)}.
\end{align}
It is therefore an atomic measure, with no density, while conditioning on
\((X,U)\).


For a feasible codebook
\(\Gamma=(-B=c_1\le\cdots\le c_M=B)\), let \(Z_\Gamma(y)\) denote the output
of stochastic endpoint interpolation followed by the metric-aware
mechanism and identity decoding, when the scalar input is \(y\).  Define
\[
 a_\Gamma(y)
 :=\E\!\left[(c_Z(y)-y)^2\right],
 \qquad
 r_\Gamma(y)
 :=\E[c_Z(y)]-y,
\]
where these expectations are only over interpolation and mechanism
randomness.

The codebook in the stated algorithm minimizes the original
uniform-coordinate population objective
\[
 \begin{aligned}
 \mathcal L(\Gamma)
 &:={\mathcal L}^{*}_{\mathrm{MA\text{-}SSTQ}}(\Gamma)
 :=\E\!\left[(Y-c_Z(Y))^2\right] \\
 &=\E_{X,U}\!\left[\frac1N\sum_{j=1}^N
 a_\Gamma(y_j(X,U))\right].
 \end{aligned}
\]
Here the expectation is over the population law of \((X,U)\), the uniform
index \(J\), and the internal scalar-mechanism randomness.  Assume that the
unconditional law \(P_Y\) has a strictly positive continuous density
\(f_Y\) on \([-B,B]\).  This density is introduced only after averaging the
conditional atomic law in \eqref{eq:cond-PU}, and
\[
 \mathcal L(\Gamma)
 =\int_{-B}^{B}a_\Gamma(y)f_Y(y)\,dy.
\]
Let \(\Gamma^*\) be a population minimizer of this objective.  Thus, for
every feasible deterministic competitor \(\Gamma'\),
\[
 \mathcal L(\Gamma^*)\le \mathcal L(\Gamma').
\]
The reconstruction error involves the distinct row-norm-weighted loss
\begin{align}\label{eq:8}
 \Lw(\Gamma)
 :=\E_{X,U}\!\left[\frac1N\sum_{j=1}^N
 w_j a_\Gamma(y_j(X,U))\right].
\end{align}
This is not the objective minimized by \(\Gamma^*\).  Nevertheless,
\(w_j\le\mu_U\) and nonnegativity of \(a_\Gamma\) imply
\begin{align}\label{eq:9}
 \Lw(\Gamma^*)
 \le\mu_U\mathcal L(\Gamma^*).
\end{align}

\textbf{2. Scalar mechanism bounds.}
For an intermediate codeword label \(V=i\), set
\[
 \beta_i:=\E[c_Z\mid V=i]-c_i,
 \qquad
 \Phi_i:=\E[(c_Z-c_i)^2\mid V=i].
\]
The continuous truncated-Laplace variable has scale \(4B/\epsilon\).
Since \(c_i\) is itself a possible nearest-codeword output,
\[
 |c_Z-c_i|\le 2|T-c_i|.
\]
The first two absolute moments of the corresponding untruncated Laplace
variable therefore give
\begin{align}\label{eq:10}
\beta_{\max}:=\max_i|\beta_i|
 \le\frac{8B}{\epsilon},
 \qquad
 \Phi_{\max}:=\max_i\Phi_i
 \le\frac{128B^2}{\epsilon^2}.
\end{align}
If \(y\in[c_k,c_{k+1}]\), stochastic endpoint interpolation is unbiased and
\(r_\Gamma(y)\) is a convex combination of \(\beta_k\) and
\(\beta_{k+1}\).  Consequently,
\begin{align}
 |r_\Gamma(y)|\le\beta_{\max}.
 \label{eq:11}
\end{align}
Writing \(S_\Gamma(y):=\E[(c_V-y)^2]\), we also have
\begin{align}
 S_\Gamma(y)=(y-c_k)(c_{k+1}-y)
\quad\text{for }y\in[c_k,c_{k+1}].
 \label{eq:12}
\end{align}
The inequality \((a+b)^2\le2a^2+2b^2\), applied to
\(c_Z-y=(c_V-y)+(c_Z-c_V)\), yields
\begin{align}
 a_\Gamma(y)\le2S_\Gamma(y)+2\Phi_{\max}.
 \label{eq:13}
\end{align}

\textbf{3. Population loss of the optimized codebook.}
First consider the uniform grid.  Its gap is
\(\Delta=2B/(M-1)\), and (12) gives
\[
 S_{\Gamma_{\mathrm{unif}}}(y)\le\frac{\Delta^2}{4}
 =\frac{B^2}{(M-1)^2}.
\]
Using \eqref{eq:13}, \eqref{eq:10}, and population optimality,
\begin{align}
 \mathcal L(\Gamma^*)
 \le
 \frac{2B^2}{(M-1)^2}
 +\frac{256B^2}{\epsilon^2}.
 \label{eq:14}
\end{align}

For the distribution-dependent competitor, define
\[
 S:=\int_{-B}^{B}f_Y(y)^{1/3}\,dy,
 \qquad
 C_f:=\frac{S^3}{4B^2}.
\]
Choose the companding partition
\(-B=c_1<\cdots<c_M=B\) so that
\begin{align}\label{eq:15}
 \int_{c_k}^{c_{k+1}}f_Y(y)^{1/3}\,dy
 =\frac{S}{M-1}.
\end{align}
Let
\[
 M_k:=\sup_{[c_k,c_{k+1}]}f_Y,
 \qquad
 m_k:=\inf_{[c_k,c_{k+1}]}f_Y.
\]
Strict positivity and continuity ensure \(m_k>0\).  From \eqref{eq:15},
\[
 c_{k+1}-c_k
 \le\frac{S}{(M-1)m_k^{1/3}}.
\]
It follows that
\begin{align}
 \begin{aligned}
 &\int_{c_k}^{c_{k+1}}
 (y-c_k)(c_{k+1}-y)f_Y(y)\,dy \\
 &\qquad\le
 M_k\frac{(c_{k+1}-c_k)^3}{6}
 \le
 \frac{S^3}{6(M-1)^3}\frac{M_k}{m_k}.
 \end{aligned}
 \label{eq:16}
\end{align}
With
\[
 R_M:=\frac1{M-1}\sum_{k=1}^{M-1}\frac{M_k}{m_k},
\]
summing \eqref{eq:16} and using \(S^3=4B^2C_f\) gives
\begin{align}
 \int_{-B}^{B}S_{\Gamma_{\mathrm{PD}}}(y)f_Y(y)\,dy
 \le
 \frac{B^2}{(M-1)^2}\left(\frac23C_fR_M\right).
 \label{eq:17}
\end{align}
Combining the better of \eqref{eq:14} and \eqref{eq:17} with \eqref{eq:13}, define
\[
 C_0^2:=
 \min\left\{1,\frac23C_fR_M\right\}
\]
and obtain
\begin{align}
 \mathcal L(\Gamma^*)
 \le
 \frac{2B^2C_0^2}{(M-1)^2}
 +\frac{256B^2}{\epsilon^2}.
 \label{eq:18}
\end{align}

\textbf{4. Population bias.}
Fix \((x,u)\), write \(y=y(x,u)\), and define
\[
 r:=
 \bigl(r_{\Gamma^*}(y_1),\ldots,r_{\Gamma^*}(y_N)\bigr)^\top.
\]
The conditional vector bias is
\[
 b(x,u)
 :=\E[\widehat X-X\mid X=x,U=u]
 =\frac dN U^\top r.
\]
Since \(\|UU^\top\|_{\rm op}=N/d\), equations \eqref{eq:10}-\eqref{eq:11} give
\[
 \begin{aligned}
 \norm{b(x,u)}^2
 &\le\frac dN\norm{r}^2
 \le d\beta_{\max}^2 \\
 &\le\frac{64dB^2}{\epsilon^2}
 =\frac{64K^2R^2}{\rho\epsilon^2}.
 \end{aligned}
\]
Jensen's inequality now yields the claimed population-bias bound:
\[
 \left\|\E[\widehat X-X]\right\|_2^2
 =\left\|\E_{X,U}[b(X,U)]\right\|_2^2
 \le\E_{X,U}\norm{b(X,U)}^2
 \le\frac{64K^2R^2}{\rho\epsilon^2}.
\]

\textbf{5. Population mean-squared error.}
For fixed \((x,u)\), let
\[
 a_j:=a_{\Gamma^*}(y_j),
 \qquad
 r_j:=r_{\Gamma^*}(y_j),
 \qquad
 P:=\frac dN UU^\top.
\]
The matrix \(P\) is an orthogonal projection.  Using
\(\widehat X=dc_{Z,J}u_J\) and
\(\E[c_{Z,j}^2]=a_j+y_j^2+2y_jr_j\), direct expansion gives the exact
conditional identity
\begin{align}
 \E\left[\norm{\widehat X-x}^2\Big| x,u\right]
 &=
 \frac{d^2}{N}\sum_{j=1}^N w_j a_j
 +\frac{d^2}{N}y^\top D_Uy-\norm{x}^2 \nonumber\\
 &\qquad
 +\frac{2d}{N}y^\top(dD_U-P)r.\label{eq:21}
\end{align}
The Kashin envelope and \eqref{eq:11} imply
\[
 \norm{y}\le\sqrt N B,
 \qquad
 \norm{r}\le\sqrt N\beta_{\max}.
\]
In addition, \(y^\top D_Uy\le B^2\sum_jw_j=NB^2\).  Moreover,
\[
\|dD_U-P\|_{\rm op}
\le d\|D_U\|_{\rm op}+\|P\|_{\rm op}
\le d\mu_U+1.
\]
Hence the last term in \eqref{eq:21}
satisfy
\[
 \left|
 \frac{2d}{N}y^\top(dD_U-P)r
 \right|
 \le2dB\beta_{\max}(d\mu_U+1)
 \le\frac{16d^2B^2}{\epsilon}
 \left(\mu_U+\frac1d\right).
\]
Also, following the same derivation as in~\eqref{eq:trick-R}, we have
\[
 \frac{d^2}{N}y^\top D_Uy-\norm{x}^2
 \le \left(\frac{K^2d}{\rho}-1\right)R^2,
\]
Taking expectation over \((X,U)\) in \eqref{eq:21}, equations \eqref{eq:8}-\eqref{eq:9} give
\begin{align}
 \E_{X,U}\!\left[\frac1N\sum_{j=1}^Nw_ja_j\right]
 =\Lw(\Gamma^*)
 \le\mu_U\mathcal L(\Gamma^*).
 \label{eq:24}
\end{align}
Combining \eqref{eq:18} and \eqref{eq:21}-\eqref{eq:24}, and using
\(d^2B^2=dK^2R^2/\rho\) and \(M=2^b\), yields
\[
 \E\norm{\widehat X-X}^2
 \le
 \frac{dK^2R^2}{\rho}
 \left(
 1+\frac{2\mu_U C_0^2}{(2^b-1)^2}
 +\frac{16(\mu_U+1/d)}{\epsilon}
 +\frac{256\mu_U}{\epsilon^2}
 \right)-R^2.
\]
This concludes the proof.

\if false
\fi
\section{Proofs for Federated Learning Convergence Analysis}\label{app:fed}

This appendix provides the  proofs for the non-convex Distributed Stochastic Gradient Descent (DSGD) convergence rate established in Theorem \ref{thm:convergence}, as well as its specific operational corollaries for our SSTQ framework.

\subsection{Proof of Theorem \ref{thm:convergence}}
The proof proceeds in four parts: decomposing the Mean Squared Error (MSE) of the global estimator, establishing the descent lemma with bias, absorbing the bias via Young's inequality, and telescoping to yield the final convergence rate.

Throughout, we let $\cF_t = \sigma\!\bigl(w_0,\; \hat{g}_{i,s} : i \in [n],\, s < t\bigr)$ denote the $\sigma$-algebra generated by the initial point and all transmitted gradients prior to step $t$. Since $w_t$ is a deterministic function of $(w_0, \hat{g}_{i,0}, \ldots, \hat{g}_{i,t-1})_{i=1}^{n}$, the iterate $w_t$ is $\cF_t$-measurable. We write $\E_t[\cdot] = \E[\cdot \mid \cF_t]$ for the conditional expectation given $\cF_t$.


\textbf{Step 1: Gradient and MSE Decomposition.} \\
Define the client-level error $X_{i,t} = \hat{g}_{i,t} - \nabla f_i(w_t)$. The global estimator error decomposes as
\[
    \hat{g}_t - \nabla F(w_t) = \frac{1}{n}\sum_{i=1}^n X_{i,t}.
\]
Applying the standard bias-variance identity $\E[\norm{Y}^2] = \E[\norm{Y - \E[Y]}^2] + \norm{\E[Y]}^2$ conditionally on $\cF_t$, with $Y = \frac{1}{n}\sum_{i=1}^n X_{i,t}$:
\begin{equation}\label{eq:bv_decomp}
    \E_t\!\left[\norm{\frac{1}{n}\sum_{i=1}^n X_{i,t}}^2\right]
    = \E_t\!\left[\norm{\frac{1}{n}\sum_{i=1}^n \bigl(X_{i,t} - \E_t[X_{i,t}]\bigr)}^2\right]
    + \norm{\E_t\!\left[\frac{1}{n}\sum_{i=1}^n X_{i,t}\right]}^2.
\end{equation}

\emph{Variance term.}
By assumption~(v), the centered variables $X_{i,t} - \E_t[X_{i,t}]$ are conditionally independent and mean-zero given $\cF_t$. Hence, for the variance of their average:
\begin{equation}\label{eq:var_sum}
    \E_t\!\left[\norm{\frac{1}{n}\sum_{i=1}^n \bigl(X_{i,t} - \E_t[X_{i,t}]\bigr)}^2\right]
    = \frac{1}{n^2}\sum_{i=1}^n \E_t\!\bigl[\norm{X_{i,t} - \E_t[X_{i,t}]}^2\bigr]
    \le \frac{1}{n^2}\sum_{i=1}^n \E_t\!\bigl[\norm{X_{i,t}}^2\bigr],
\end{equation}
where the inequality uses $\mathrm{Var}(Z) \le \E[\norm{Z}^2]$ for any random variable $Z$.

\emph{Bias term.}
By linearity, the squared bias equals
\[
    \norm{\E_t\!\left[\frac{1}{n}\sum_{i=1}^n X_{i,t}\right]}^2
    = \norm{\E_t[\hat{g}_t] - \nabla F(w_t)}^2
    = \norm{b_t}^2
    \le \zeta^2 \quad \text{a.s.},
\]
by assumption~(iv).

Combining \eqref{eq:bv_decomp}--\eqref{eq:var_sum}:
\begin{equation}\label{eq:mse_intermediate}
    \E_t\!\left[\norm{\hat{g}_t - \nabla F(w_t)}^2\right]
    \le \frac{1}{n^2}\sum_{i=1}^n \E_t\!\bigl[\norm{X_{i,t}}^2\bigr] + \zeta^2.
\end{equation}

\emph{Bounding the local MSE.}
We bound each $\E_t[\norm{X_{i,t}}^2]$ by splitting the compression and stochastic errors. Using $\norm{A+B}^2 \le 2\norm{A}^2 + 2\norm{B}^2$ (which holds for all vectors $A, B$ without any independence requirement):
\begin{align}
    \E_t\!\bigl[\norm{X_{i,t}}^2\bigr]
    &= \E_t\!\bigl[\norm{(\hat{g}_{i,t} - g_{i,t}) + (g_{i,t} - \nabla f_i(w_t))}^2\bigr] \nonumber \\
    &\le 2\,\E_t\!\bigl[\norm{\hat{g}_{i,t} - g_{i,t}}^2\bigr]
       + 2\,\E_t\!\bigl[\norm{g_{i,t} - \nabla f_i(w_t)}^2\bigr] \nonumber \\
    &\le 2\omega G^2 + 2\sigma^2
    = 2(\omega G^2 + \sigma^2), \label{eq:local_mse}
\end{align}
where the last line applies assumptions~(ii) and~(iii).

Substituting \eqref{eq:local_mse} into \eqref{eq:mse_intermediate}:
\begin{equation}\label{eq:global_mse}
    \E_t\!\left[\norm{\hat{g}_t - \nabla F(w_t)}^2\right]
    \le \frac{2(\omega G^2 + \sigma^2)}{n} + \zeta^2.
\end{equation}

\emph{Second-moment bound.}
We additionally bound the conditional second moment of the update direction, which will be needed in Step~2. By the same inequality $\norm{A+B}^2 \le 2\norm{A}^2 + 2\norm{B}^2$:
\begin{align}
    \E_t\!\bigl[\norm{\hat{g}_t}^2\bigr]
    &= \E_t\!\bigl[\norm{(\hat{g}_t - \nabla F(w_t)) + \nabla F(w_t)}^2\bigr] \nonumber \\
    &\le 2\,\E_t\!\bigl[\norm{\hat{g}_t - \nabla F(w_t)}^2\bigr] + 2\norm{\nabla F(w_t)}^2 \nonumber \\
    &\le 2\!\left(\frac{2(\omega G^2 + \sigma^2)}{n} + \zeta^2\right) + 2\norm{\nabla F(w_t)}^2, \label{eq:second_moment}
\end{align}
where the last step substitutes \eqref{eq:global_mse}. Note that $\nabla F(w_t)$ is $\cF_t$-measurable and thus passes through $\E_t[\cdot]$ as a constant.

\textbf{Step 2: The Descent Lemma with Bias.} \\
By $L$-smoothness (assumption~(i)) and the update rule $w_{t+1} = w_t - \eta\,\hat{g}_t$:
\begin{equation}\label{eq:descent_pointwise}
    F(w_{t+1})
    \le F(w_t)
    - \eta\,\langle \nabla F(w_t),\, \hat{g}_t \rangle
    + \frac{\eta^2 L}{2}\,\norm{\hat{g}_t}^2.
\end{equation}
Taking the conditional expectation $\E_t[\cdot]$ of both sides (noting $F(w_t)$ and $\nabla F(w_t)$ are $\cF_t$-measurable):
\begin{align}
    \E_t[F(w_{t+1})]
    &\le F(w_t) - \eta\,\bigl\langle \nabla F(w_t),\, \E_t[\hat{g}_t] \bigr\rangle + \frac{\eta^2 L}{2}\,\E_t\!\bigl[\norm{\hat{g}_t}^2\bigr] \nonumber \\
    &= F(w_t) - \eta\,\bigl\langle \nabla F(w_t),\, \nabla F(w_t) + b_t \bigr\rangle + \frac{\eta^2 L}{2}\,\E_t\!\bigl[\norm{\hat{g}_t}^2\bigr] \nonumber \\
    &= F(w_t) - \eta\,\norm{\nabla F(w_t)}^2 - \eta\,\langle \nabla F(w_t),\, b_t \rangle + \frac{\eta^2 L}{2}\,\E_t\!\bigl[\norm{\hat{g}_t}^2\bigr]. \label{eq:descent_conditional}
\end{align}

\textbf{Step 3: Absorbing the Bias via Young's Inequality.} \\
We bound the inner-product bias term using the Cauchy--Schwarz inequality followed by Young's inequality ($ab \le \tfrac{a^2}{2} + \tfrac{b^2}{2}$ for $a,b \ge 0$):
\begin{equation}\label{eq:young}
    -\langle \nabla F(w_t),\, b_t \rangle
    \le \norm{\nabla F(w_t)}\,\norm{b_t}
    \le \frac{1}{2}\norm{\nabla F(w_t)}^2 + \frac{1}{2}\norm{b_t}^2
    \le \frac{1}{2}\norm{\nabla F(w_t)}^2 + \frac{\zeta^2}{2},
\end{equation}
where the last inequality uses $\norm{b_t}^2 \le \zeta^2$ a.s.\ (assumption~(iv)).

Substituting \eqref{eq:young} and the second-moment bound \eqref{eq:second_moment} into \eqref{eq:descent_conditional}:
\begin{align}
    \E_t[F(w_{t+1})]
    &\le F(w_t)
    -\frac{\eta}{2}\norm{\nabla F(w_t)}^2 + \frac{\eta\,\zeta^2}{2}
    + \frac{\eta^2 L}{2}\!\left[2\norm{\nabla F(w_t)}^2 + 2\!\left(\frac{2(\omega G^2 + \sigma^2)}{n} + \zeta^2\right)\right] \nonumber \\
    &= F(w_t)
    - \underbrace{\left(\frac{\eta}{2} - \eta^2 L\right)}_{\text{gradient coefficient}}\norm{\nabla F(w_t)}^2
    + \frac{\eta\,\zeta^2}{2}
    + \eta^2 L\!\left(\frac{2(\omega G^2 + \sigma^2)}{n} + \zeta^2\right). \label{eq:descent_combined}
\end{align}

\textbf{Step 4: Telescoping and Convergence.} \\
\emph{Learning rate constraint.}
We require $\eta \le \frac{1}{4L}$, which ensures $\eta^2 L \le \frac{\eta}{4}$ and hence:
\begin{equation}\label{eq:lr_constraint}
    \frac{\eta}{2} - \eta^2 L \ge \frac{\eta}{2} - \frac{\eta}{4} = \frac{\eta}{4} > 0.
\end{equation}
Substituting into \eqref{eq:descent_combined}:
\begin{equation}\label{eq:descent_clean}
    \E_t[F(w_{t+1})]
    \le F(w_t) - \frac{\eta}{4}\,\norm{\nabla F(w_t)}^2
    + \frac{\eta\,\zeta^2}{2}
    + \eta^2 L\!\left(\frac{2(\omega G^2 + \sigma^2)}{n} + \zeta^2\right).
\end{equation}

\emph{Telescoping.}
Rearranging \eqref{eq:descent_clean}:
\[
    \frac{\eta}{4}\,\norm{\nabla F(w_t)}^2
    \le F(w_t) - \E_t[F(w_{t+1})]
    + \frac{\eta\,\zeta^2}{2}
    + \eta^2 L\!\left(\frac{2(\omega G^2 + \sigma^2)}{n} + \zeta^2\right).
\]
Taking the full (unconditional) expectation on both sides and applying the tower property $\E[\E_t[\cdot]] = \E[\cdot]$:
\[
    \frac{\eta}{4}\,\E\!\left[\norm{\nabla F(w_t)}^2\right]
    \le \E[F(w_t)] - \E[F(w_{t+1})]
    + \frac{\eta\,\zeta^2}{2}
    + \eta^2 L\!\left(\frac{2(\omega G^2 + \sigma^2)}{n} + \zeta^2\right).
\]
Summing over $t = 0, 1, \ldots, T-1$ (the right-hand side telescopes):
\[
    \frac{\eta}{4}\sum_{t=0}^{T-1}\E\!\left[\norm{\nabla F(w_t)}^2\right]
    \le F(w_0) - \E[F(w_T)]
    + \frac{T\eta\,\zeta^2}{2}
    + T\eta^2 L\!\left(\frac{2(\omega G^2 + \sigma^2)}{n} + \zeta^2\right).
\]
Using $\E[F(w_T)] \ge F^*$ (Assumption~\ref{ass:optimization}(i)) and dividing both sides by $\frac{\eta T}{4}$:
\begin{equation}\label{eq:avg_grad}
    \frac{1}{T}\sum_{t=0}^{T-1}\E\!\left[\norm{\nabla F(w_t)}^2\right]
    \le \frac{4\Delta_0}{\eta T}
    + 2\zeta^2
    + 4\eta L\!\left(\frac{2(\omega G^2 + \sigma^2)}{n} + \zeta^2\right),
\end{equation}
where $\Delta_0 \coloneqq F(w_0) - F^*$.

\emph{Optimal learning rate.}
Define the noise-plus-bias parameter $V \coloneqq \frac{2(\omega G^2 + \sigma^2)}{n} + \zeta^2$. The bound \eqref{eq:avg_grad} has the form $\frac{4\Delta_0}{\eta T} + 2\zeta^2 + 4\eta L V$. We set $\eta = \min\!\bigl\{\frac{1}{4L},\; \sqrt{\frac{\Delta_0}{LVT}}\bigr\}$ and consider both cases.

\emph{Case 1: $T \ge T_0 \coloneqq 16 L \Delta_0 / V$.} In this regime, $\eta^* = \sqrt{\Delta_0 / (LVT)} \le \frac{1}{4L}$, and substituting into \eqref{eq:avg_grad}:
\begin{align}
    \frac{1}{T}\sum_{t=0}^{T-1}\E\!\left[\norm{\nabla F(w_t)}^2\right]
    &\le \frac{4\Delta_0}{\sqrt{\Delta_0/(LVT)} \cdot T} + 2\zeta^2 + 4\sqrt{\frac{\Delta_0}{LVT}} \cdot LV \nonumber \\
    &= \frac{8\sqrt{\Delta_0 L V}}{\sqrt{T}} + 2\zeta^2. \label{eq:rate_with_V}
\end{align}

\emph{Case 2: $T < T_0$.} Here $\eta = \frac{1}{4L}$, and substituting into \eqref{eq:avg_grad}:
\begin{align}
    \frac{1}{T}\sum_{t=0}^{T-1}\E\!\left[\norm{\nabla F(w_t)}^2\right]
    &\le \frac{16L\Delta_0}{T} + 2\zeta^2 + V. \label{eq:burn_in}
\end{align}
For $T < T_0$, we have $\frac{8\sqrt{\Delta_0 L V}}{\sqrt{T}} > \frac{8\sqrt{\Delta_0 L V}}{\sqrt{T_0}} = \frac{8\sqrt{\Delta_0 L V} \cdot \sqrt{V}}{4\sqrt{L\Delta_0}} = 2V$, so the $V$ term in the burn-in bound is dominated by the square-root term in Case 1 evaluated at $T$. Therefore, the unified bound valid for all $T \ge 1$ is:
\begin{equation}
    \min_{0 \le t \le T-1}\E\!\left[\norm{\nabla F(w_t)}^2\right]
    \le \frac{16L\Delta_0}{T} + \frac{8\sqrt{\Delta_0 L V}}{\sqrt{T}} + 2\zeta^2.
\end{equation}

\emph{Expanding $V$ and simplifying.}
Recall $V = \frac{2(\omega G^2 + \sigma^2)}{n} + \zeta^2$. Thus $\sqrt{V} \le \sqrt{\frac{2(\omega G^2+\sigma^2)}{n}} + \sqrt{\zeta^2}$, where we have used $\sqrt{a+b} \le \sqrt{a} + \sqrt{b}$ for $a,b \ge 0$. Substituting:
\begin{align}
    \frac{8\sqrt{\Delta_0 L V}}{\sqrt{T}}
    &\le \frac{8\sqrt{\Delta_0 L}}{\sqrt{T}}\left(\sqrt{\frac{2(\omega G^2+\sigma^2)}{n}} + \zeta\right) \nonumber \\
    &= \frac{8\sqrt{2\,\Delta_0 L}\cdot\sqrt{\omega G^2+\sigma^2}}{\sqrt{nT}} + \frac{8\sqrt{\Delta_0 L}\cdot\zeta}{\sqrt{T}}. \label{eq:expanded}
\end{align}

Combining the unified bound with \eqref{eq:expanded}, and using $\min_{t} \le \frac{1}{T}\sum_t$:
\begin{equation}\label{eq:full_rate}
    \min_{0 \le t \le T-1}\E\!\left[\norm{\nabla F(w_t)}^2\right]
    \le \frac{16L\Delta_0}{T} + \frac{8\sqrt{2\,\Delta_0 L}\cdot\sqrt{\omega G^2+\sigma^2}}{\sqrt{nT}}
    + \frac{8\sqrt{\Delta_0 L}\cdot\zeta}{\sqrt{T}}
    + 2\zeta^2.
\end{equation}

\emph{Asymptotic simplification.}
For $T \ge T_0 = 16L\Delta_0 / V$, the $\frac{16L\Delta_0}{T}$ burn-in term is dominated by $\frac{8\sqrt{\Delta_0 L V}}{\sqrt{T}}$. For the $\frac{8\sqrt{\Delta_0 L}\cdot\zeta}{\sqrt{T}}$ term, we apply the AM--GM inequality $\frac{\zeta}{\sqrt{T}} \le \frac{\zeta^2}{2} + \frac{1}{2T}$, giving $\frac{8\sqrt{\Delta_0 L}\cdot\zeta}{\sqrt{T}} \le 4\sqrt{\Delta_0 L}\,\zeta^2 + \frac{4\sqrt{\Delta_0 L}}{T}$. Since $T \ge 16L\Delta_0/V \ge 1$, the second term is $\cO(1/T)$ and absorbed into the burn-in. The first term is absorbed into $\cO(\zeta^2)$ with the problem-dependent constant $4\sqrt{\Delta_0 L}$. Therefore, the asymptotic rate (for $T \ge 16L\Delta_0/V$) is:
\begin{equation}
    \min_{0 \le t \le T-1}\E\!\left[\norm{\nabla F(w_t)}^2\right]
    \le \cO\!\left(\frac{\sqrt{\omega G^2+\sigma^2}}{\sqrt{nT}} + \zeta^2\right),
\end{equation}
where the $\cO(\cdot)$ hides dependence on the problem constants $\Delta_0$ and $L$. This completes the proof.

\subsection{Proof of Corollary \ref{cor:convergence_flat}}
Under Flat Randomized Response, the SSTQ mechanism employs the unbiasing scalar $\frac{1}{p-q}$ on the server. As established in Theorem \ref{thm:sstq_ldp_unbiased}, this renders the global estimator perfectly unbiased in expectation: $b_t = \mathbf{0}$, which mathematically enforces $\zeta = 0$. 

By Theorem \ref{thm:miracle} (or more precisely, by Theorem \ref{thm:one_third_bound} when $M \ge (4(e^\epsilon-1))^{1/3}+1$), the total variance multiplier under Flat RR is bounded by $\omega = \cO((dK^2/\rho) [1+2^b / (\epsilon \wedge \epsilon^2)])$. Substituting $\zeta = 0$ and this variance multiplier directly into the general convergence bound of Theorem \ref{thm:convergence} gives:
\begin{equation}
    \min_{0 \le t \le T-1} \E[\norm{\nabla F(w_t)}^2] \le \cO\left( \frac{1}{T} + \frac{\sqrt{ \frac{dK^2}{\rho} [1+ 2^b/ (\epsilon \wedge \epsilon^2)] G^2 + \sigma^2 }}{\sqrt{nT}} \right).
\end{equation}
Because the bias is exactly zero, the gradient norm converges to zero as $T \to \infty$.

\subsection{Proof of Corollary \ref{cor:convergence_metric}}
The Metric-Aware Laplace framework trades absolute unbiasedness to address the $\cO(2^b)$ Flat RR variance inflation. 
Under the max-gap and row-coherence assumptions, Theorem~\ref{thm:SSTQ-MA} gives
$\omega = \cO((dK^2/\rho)(1 + \epsilon^{-1} + \epsilon^{-2}))$ and
    $\zeta^2 \le 64K^2G^2/(\rho \epsilon^2)$. Using these bounds along with the bound established by Theorem \ref{thm:convergence}, we get:
\begin{equation}
    \min_{0 \le t \le T-1} \E[\norm{\nabla F(w_t)}^2] \le \cO\left( \frac{1}{T} + \frac{\sqrt{ (dK^2/\rho) G^2 (1 + \epsilon^{-1} + \epsilon^{-2})   + \sigma^2 }}{\sqrt{nT}} + \frac{K^2G^2}{\rho \epsilon^2} \right).
\end{equation}
This confirms that while the optimization trajectory converges to a non-zero stationary neighborhood  dictated by $\zeta^2$, the  size of this neighborhood is explicitly independent of the massive parameter scale $d$. This algebraic decoupling  breaks the dimensionality curse, ensuring that the gradient trajectory remains highly robust and computationally viable even for extremely large neural architectures.

\section{Experimental Details}\label{app:exp-details}

To efficiently implement the high-dimensional transformations required by SQKR-style constructions, we use a randomized partial Hadamard transform, which enables computation over an overcomplete frame of dimension $N = 2^{\lceil\log_2(2.5\,d)\rceil}$ in $O(N\log N)$ time without explicit dense matrix storage.

For the Kashin representation, we apply the exact-reconstruction variant of the iterative Lyubarskii--Vershynin algorithm with parameter $K=2$. 
We run $r=30$ truncated iterations at nominal level $KC/\sqrt N$, then add the exact residual correction.  If a certified residual contraction $\|e_r\|_2\le\eta^rC$ holds, the correction $Ue_r$ can enlarge a coordinate by at most $\sqrt{\mu_U}\eta^rC$, so the implemented envelope is
\[
B_r=\frac{KC}{\sqrt N}+\sqrt{\mu_U}\,\eta^rC .
\]
For the randomized partial Hadamard frame with redundancy $N/d \approx 2.5$, we have $\eta \approx 0.4$, giving a residual of order $\eta^{30} \approx 10^{-12}$ after $30$ iterations---well below machine precision. Consequently, the untruncated final coefficients have magnitude at most $\eta^{30} \cdot \sqrt{\mu_U}C$, which is negligible, and the effective Kashin level remains $K = 2$ to numerical precision.

On the server side, after averaging the $W$ privatized messages, we apply gradient norm clipping with a threshold of $10.0$ to the aggregate.
Because clipping is a nonlinear operation, it can introduce a small bias even when the underlying per-client estimator is unbiased (see Remark~7.1 in the main text).
In practice the cap activates infrequently: the client-side clipping at norm~$C$ combined with the frame scaling $B = K/\sqrt{N}$ keeps the expected aggregate norm well below~$10.0$, so the induced bias is negligible.

\begin{remark}[Effect of Learning Rate on High-Variance Mechanisms]\label{rem:lr}
A natural question is whether reducing the learning rate can compensate for the high variance introduced by certain privatization mechanisms.
The convergence bound in Theorem~\ref{thm:convergence} provides a precise answer.
With step size~$\eta$, the averaged gradient norm is bounded by $\frac{4\Delta_0}{\eta T} + 2\zeta^2 + 4\eta L V$, where $V = \frac{2(\omega G^2 + \sigma^2)}{n} + \zeta^2$.
Reducing~$\eta$ shrinks the variance term $4\eta L V$ but inflates the optimization term $\frac{4\Delta_0}{\eta T}$; the optimal trade-off $\eta^* = \sqrt{\Delta_0/(LVT)}$ yields a rate of $\mathcal{O}(\sqrt{V/T})$.
Consequently, a mechanism with $\omega$~times larger variance multiplier requires $\Theta(\omega)$~times more communication rounds to reach the same accuracy.
For vqSGD, whose variance scales as $\Theta(d^3/\epsilon^2)$, this translates to an impractical number of rounds in high dimensions, regardless of step-size tuning.
\end{remark}

For the optimized codebook variant of SSTQ (Flat-RR), we solve the loss minimization problem
\[
  \mathcal{L}_{\mathrm{SSTQ}}(\Gamma)
  \;=\;
  \sum_{k=1}^{M-1}\frac{(c_{k+1}-c_k)^3}{12B}
  \;+\;
  \frac{1}{e^\epsilon - 1}\sum_{i=1}^{M} c_i^2
\]
subject to $c_1=-B,\; c_M=B,\; c_1 \le c_2 \le \cdots \le c_M,\; \textstyle\sum_i c_i=0$, using Sequential Least-Squares Programming (SLSQP).
The codebook is precomputed once per configuration $(d,b,\epsilon)$ using the worst-case bound $B = K\,C/\sqrt{N}$, which is constant across all clipped gradients.

All experiments were implemented in Python using NumPy and SciPy, and were run on a single CPU.


\section{Additional Experiments}\label{app:add-exp}

We conduct additional experiments designed to directly test the variance bounds, communication efficiency, and downstream model utility across varying dimensions~$d$ and codebook sizes.
In these experiments, we compare seven methods:
Clean (no privacy), PrivUnit, vqSGD, SQKR,
SSTQ (Flat-RR with optimized codebook),
SSTQ (Flat-RR with uniform codebook),
and SSTQ (Metric-Aware with uniform codebook).

\paragraph{Empirical Variance Scaling (Theorem~A.1 vs.\ Theorem~3.2).}
In this experiment, we isolate the client-to-server gradient transmission step and focus on computing the raw Mean Squared Error (MSE) between the true vector and the reconstructed estimator.
We generate a single gradient vector as $g \sim \mathcal{N}(0,I_d)$ and normalize it to have norm $C = 0.2$, to ensure bounded sensitivity under LDP.
We pass the same vector~$g$ into each algorithm, which quantizes the vector, adds LDP noise, and reconstructs an estimated vector~$\hat{g}$.
We sweep the dimension~$d$ over $\{100,\, 500,\, 1000,\, 5000,\, 10000\}$ and plot the MSE $\mathbb{E}[\|g - \hat{g}\|_2^2]$ versus~$d$ on a log-log scale, by averaging over 1000 independent trials per dimension.
All methods use per-round privacy parameter $\epsilon = 3$ and bit budget $b = 8$ (codebook size $M = 2^8 = 256$).
The 95\% confidence intervals are sufficiently narrow that they are not visually discernible on the log-log scale, reflecting the reliability of the Monte Carlo estimates at this sample size.

\begin{figure}[ht]
  \centering
  \includegraphics[width=0.75\textwidth]{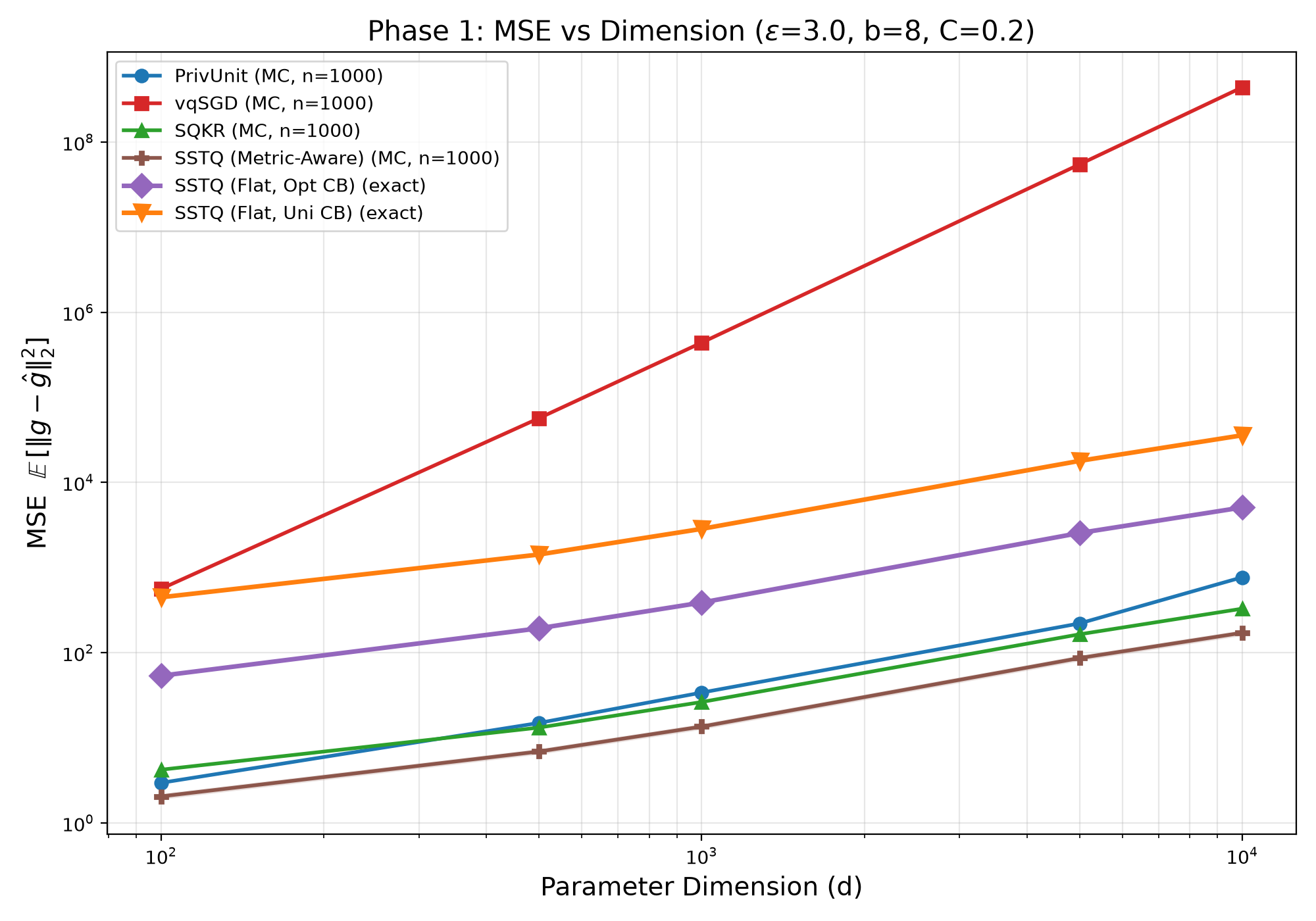}
  \caption{MSE versus dimension for all methods (per-round $\epsilon = 3$, $b = 8$, $M = 256$, $C = 0.2$).
  Flat-RR variants are computed exactly (no Monte Carlo); other methods use 1000 MC trials with 95\% confidence bands.}
  \label{fig:mse-vs-dim}
\end{figure}

The results, reported in Figure~\ref{fig:mse-vs-dim}, confirm the theoretical predictions.
vqSGD exhibits devastating cubic $O(d^3/\epsilon^2)$ scaling, with MSE growing from $\approx 561$ at $d=100$ to $\approx 4.4 \times 10^8$ at $d=10{,}000$---an increase consistent with the cubic ratio.
In contrast, PrivUnit, SQKR, and SSTQ (Metric-Aware) all exhibit linear scaling in~$d$, consistent with the $O(d/\epsilon^2)$ bound of Theorem~3.2.
Among these, SSTQ (Metric-Aware) achieves the lowest MSE across all dimensions ($\approx 2.0$ at $d=100$ and $\approx 171$ at $d=10{,}000$), followed by SQKR ($\approx 4.2$ to $\approx 328$) and PrivUnit ($\approx 2.9$ to $\approx 767$).

The two Flat-RR variants exhibit substantially higher MSE, now revealed by the exact second-moment computation.
SSTQ (Flat-RR, uniform codebook) has MSE ranging from $\approx 446$ at $d=100$ to $\approx 35{,}746$ at $d=10{,}000$, reflecting the $O(M^2)$ variance inflation from the RR debiasing scalar at large $M = 256$.
SSTQ (Flat-RR, optimized codebook) reduces this to $\approx 53$ at $d=100$ and $\approx 5{,}072$ at $d=10{,}000$---a $7$--$8\times$ improvement over the uniform variant, confirming that codebook optimization mitigates but does not eliminate the RR variance.
Both Flat-RR variants still incur MSE far above the Metric-Aware variant (by $27$--$30\times$), demonstrating the decisive advantage of the Metric-Aware Discrete Laplace mechanism in eliminating the exponential $2^b$ dependence.

\paragraph{Comparison of LDP mechanisms for linear regression across varying dimensions.}
We study the convergence of distributed stochastic gradient descent (SGD) on a synthetic least-squares problem with objective
\[
  L(\theta) = \frac{1}{2n}\|A\theta - b\|_2^2\,.
\]
We denote by $\theta^* \in \mathbb{R}^d$ the ground-truth parameter.
The design matrix $A \in \mathbb{R}^{n \times d}$ has entries drawn i.i.d.\ from $\mathcal{N}(0, 1/d)$, ensuring stable conditioning across dimensions.
We generate $\theta^* \sim \mathcal{N}(0, I_d)$ and set $b = A\theta^*$, yielding a realizable model.

We simulate a federated setting with $n = 4{,}000$ samples distributed uniformly across $W = 200$ workers.
Each worker computes local gradients on batches of size~$20$ over $T = 100$ communication rounds.
Gradients are clipped to $\|g\|_2 \le C = 1$, to ensure bounded sensitivity under LDP. Then privatized and transmitted to the server, which performs aggregation and updates using a constant step size $\eta = 0.5$.
Performance is measured by the squared estimation error $\|\theta_t - \theta^*\|_2^2$.

All private methods operate under pure $\epsilon$-LDP with per-round $\epsilon = 3$.
We first evaluate convergence at bit budget $b = 4$ across dimensions $d \in \{100, 200, 500\}$ (Figure~\ref{fig:sgd-b4}).
For methods requiring frame representations (SSTQ and SQKR), we use an overcomplete frame of dimension $N = 2^{\lceil \log_2(2.5 d) \rceil}$.

\begin{figure}[ht]
  \centering
  \includegraphics[width=\textwidth]{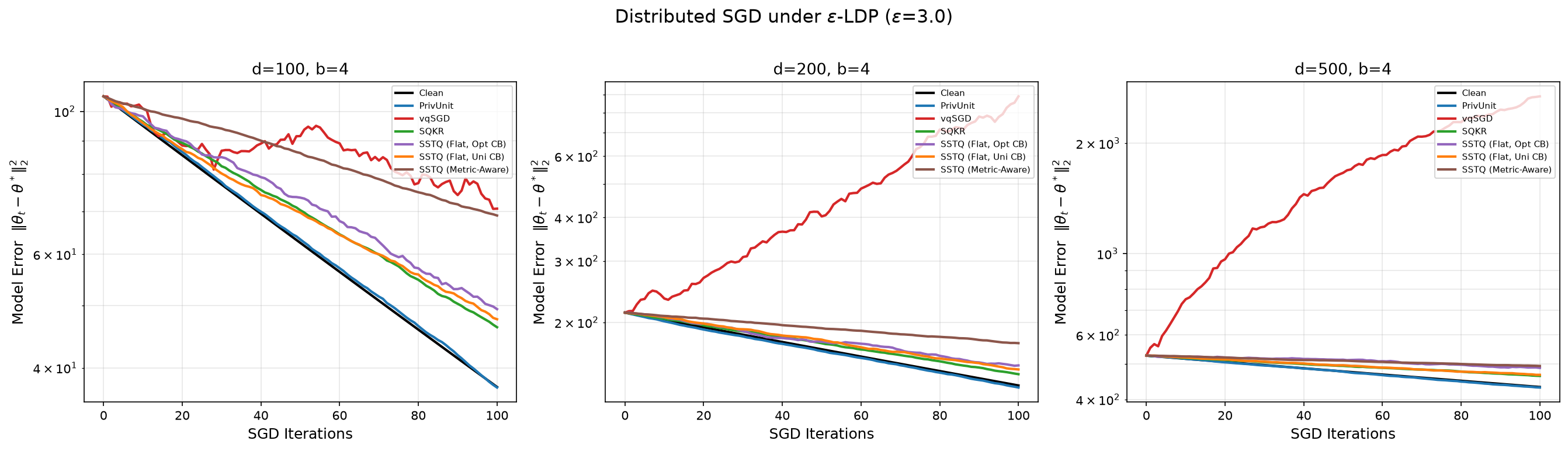}
  \caption{Distributed SGD convergence under $\epsilon$-LDP ($\epsilon=3$, $b=4$).
  Comparison of the two variants of our proposed framework---SSTQ (Flat-RR) and SSTQ (Metric-Aware)---against other vector quantization algorithms across $d \in \{100, 200, 500\}$.}
  \label{fig:sgd-b4}
\end{figure}

Figure~\ref{fig:sgd-b4} illustrates distributed least-squares SGD under pure $\epsilon$-LDP ($\epsilon = 3$).
Consistent with the observations in Section~\ref{sec:nn-experiments}, vqSGD exhibits unstable optimization behavior that becomes more pronounced as the dimension increases: the model error \emph{increases} monotonically, reaching $\approx 2685$ at $d=500$ compared to $\approx 432.6$ for the clean baseline.
While reducing the learning rate can stabilize individual updates, by Remark~\ref{rem:lr} this trades variance reduction against slower optimization: the $\Theta(d^3/\epsilon^2)$ variance multiplier of vqSGD would require a proportionally larger number of communication rounds to converge, rendering it impractical in high dimensions.
In contrast, the 1-sparse structure of SSTQ mitigates this dependence on the ambient dimension, resulting in stable and efficient convergence.

At $b=4$ (moderate codebook size $M = 16$), the Flat-RR variants converge closely to the clean baseline, with SSTQ (Flat-RR, uniform codebook) achieving final error $\approx 132.8$ at $d=200$ versus $\approx 131.7$ for the clean update---a modest overhead.
SSTQ (Flat-RR, optimized codebook) achieves final error $\approx 140.5$ at $d=200$.
The Metric-Aware variant converges to a slightly higher error floor of $\approx 162$ at $d=200$, consistent with the non-zero bias $\zeta^2$ introduced by the identity decoder (Corollary~\ref{cor:convergence_metric}).
PrivUnit tracks the clean baseline most closely ($\approx 130.0$), as it incurs no quantization overhead.

\paragraph{Comparison of SSTQ variants across codebook sizes.}
To isolate the effect of codebook size on the Flat-RR variance, we fix $d = 200$ and compare convergence at $b=4$ ($M=16$) versus $b=8$ ($M=256$) in Figure~\ref{fig:sgd-b4-vs-b8}.

\begin{figure}[ht]
  \centering
  \begin{minipage}{0.48\textwidth}
    \centering
    \includegraphics[width=\textwidth]{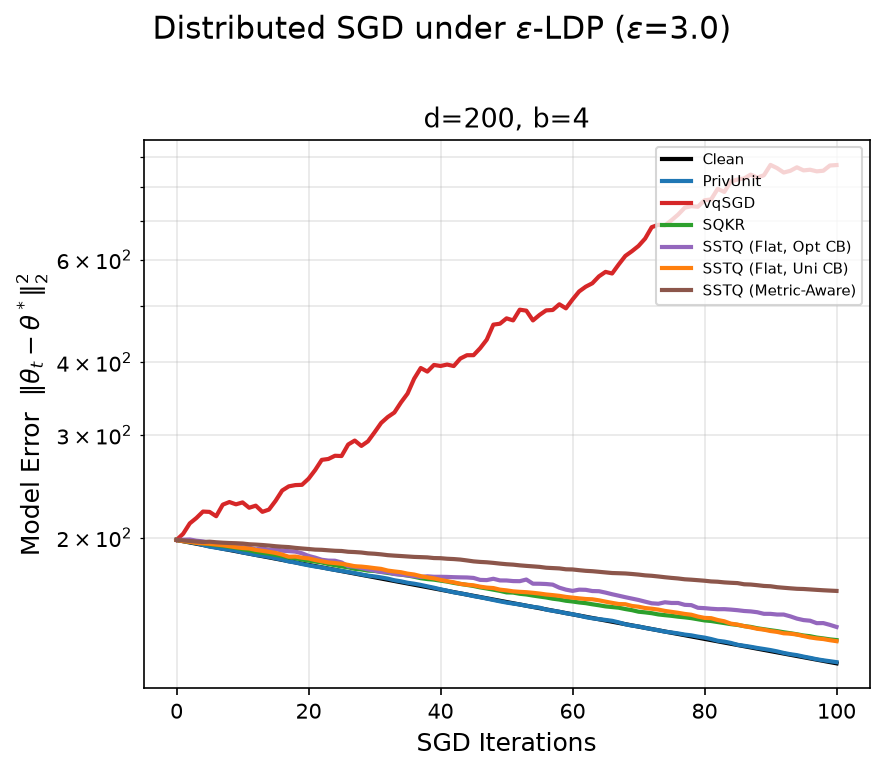}
  \end{minipage}
  \hfill
  \begin{minipage}{0.48\textwidth}
    \centering
    \includegraphics[width=\textwidth]{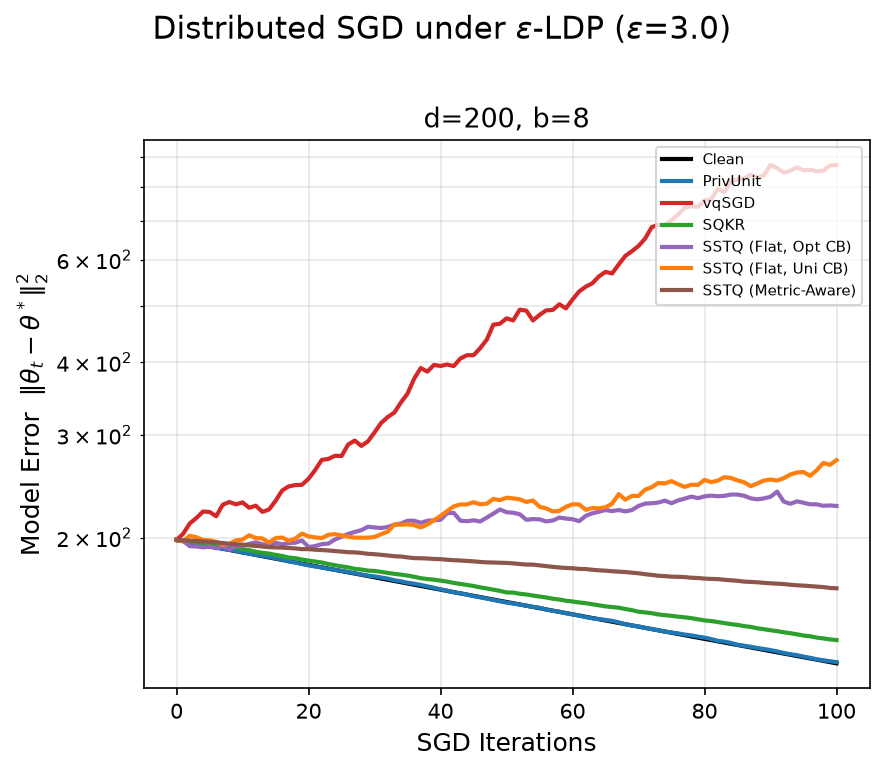}
  \end{minipage}
  \caption{Flat-RR explosion versus Metric-Aware stability at $d=200$, $\epsilon = 3$.
  \textbf{Left:} Low bit-width ($b=4$, $M=16$). \textbf{Right:} Dense bit-width ($b=8$, $M=256$).}
  \label{fig:sgd-b4-vs-b8}
\end{figure}

As shown in Figure~\ref{fig:sgd-b4-vs-b8}, increasing the codebook size from $M=16$ to $M=256$ causes a dramatic degradation in both Flat-RR variants.
SSTQ (Flat-RR, uniform codebook) sees its final error grow from $\approx 132.8$ to $\approx 271.6$ ($2.0\times$ increase), while SSTQ (Flat-RR, optimized codebook) grows from $\approx 140.5$ to $\approx 226.7$ ($1.6\times$ increase).
This is consistent with the predicted $O(M^2)$ scaling of the unbiasing factor: the debiasing scalar $1/(p-q) = (e^\epsilon + M - 1)/(e^\epsilon - 1)$ grows from $\approx 1.8$ at $M=16$ to $\approx 14.4$ at $M=256$, inflating the variance quadratically.

In stark contrast, SSTQ (Metric-Aware) remains essentially unchanged, with final error growing only marginally from $\approx 162.0$ to $\approx 163.7$.
This stability arises because the Metric-Aware Discrete Laplace mechanism does not require the uniform debiasing correction---its noise distribution concentrates naturally around the true quantization level regardless of $M$.
This behavior highlights the limitations of Flat-RR in dense quantization regimes and motivates the use of Metric-Aware estimators for maintaining stable performance at larger codebook sizes.

\end{document}